\documentclass[letterpaper]{article}

\usepackage[style=nature, backend=biber, sorting=none]{biblatex}
\usepackage{breakurl}
\usepackage[margin=1.25in]{geometry}
\usepackage[toc,page]{appendix}

\def\True{1}
\def\distribution{1}
\ifx\distribution\True
	
	\def\appendixon{1}
	\def\figureson{1}
	\newcommand{\samples}{500}
	\def\separatorpages{0}
	
	\else 
	\def\appendixon{1}
	\def\figureson{1}
	\newcommand{\samples}{25}
	\def\separatorpages{0}
\fi

\def\prepublication{0}
\ifx\prepublication\True
	\def\compilationheader{1}
	
	\else
	\def\compilationheader{0}
\fi

\newcommand{\ifdist}[2]{
	\ifx\distribution\True
	#1
	\else
	#2
	\fi
}

\newcommand{\separatorpage}{
	\ifx\separatorpages\True
		\newpage
	\fi	
}

\ifx\compilationheader\True
	\usepackage{datetime2}
	\usepackage{fancyhdr}
\fi

\usepackage{microtype}
\usepackage{authblk}

\usepackage{amsmath}
\usepackage{amssymb}
\usepackage{centernot}
\DeclareMathAlphabet{\mathpzc}{OT1}{pzc}{m}{it}

\usepackage{enumerate}

\usepackage{algorithm}
\usepackage[noend]{algpseudocode}

\usepackage{array}

\usepackage{amsthm}
\theoremstyle{remark}
\newtheorem{remark}{Remark}[section]
\newtheorem{example}{Example}[section]
\newtheorem{corollary}{Corollary}[section]

\theoremstyle{definition}
\newtheorem{definition}{Definition}[section]
\newtheorem{lemma}{Lemma}[section]
\newtheorem{theorem}{Theorem}[section]

\newcommand{\action}{\mathpzc{a}}
\newcommand{\actionSet}{\mathpzc{A}}
\newcommand{\behavior}{B}

\newcommand{\agentdefaut}{a}
\newcommand{\agentun}{a}
\newcommand{\agentdeux}{b}
\newcommand{\agenttrois}{c}

\newcommand{\agentSet}{A}

\newcommand{\TVD}{\mathop{\text{TVD}}}

\renewcommand{\qedsymbol}{\textbf{Q.E.D.}}

\newcommand{\expectation}{\mathop{\mathbb{E}}}

\newcommand{\excoef}{\rho}

\DeclareMathOperator{\Exists}{\exists}
\DeclareMathOperator{\Forall}{\forall}

\usepackage{booktabs,array}
\usepackage{calc}
\newcommand{\vdotsSpace}{2.7em}

\newcommand{\rb}[2]{%
	\raisebox{-\ht0}{\raisebox{-\height}{%
			$\left.\vphantom{\begin{array}{c}#1\end{array}}\right\rbrace#2$%
	}}%
}

\numberwithin{equation}{subsection}
\usepackage{csquotes}

\usepackage{pgfplots}
\usetikzlibrary{math,arrows,positioning,shapes,fit,calc,intersections}
\usepgfplotslibrary{fillbetween}

\newcommand{\figuresize}{0.75\linewidth}

\usepackage{xcolor}
\definecolor{colorone}{HTML}{9ECAE1}
\definecolor{colortwo}{HTML}{3182BD}
\newcommand{\colorone}{colorone}  
\newcommand{\colortwo}{colortwo} 
\newcommand{\colorthree}{black}  
\newcommand{\colorfour}{black}   

\usepackage{footmisc}

\usepackage[colorlinks, linkcolor=blue, citecolor=blue, urlcolor=blue]{hyperref} 

\newcommand{\appref}[1]{\hyperref[#1]{Appendix~\ref*{#1}}}

	\usepackage{verbatim}

	\usepackage{totcount}
	\newtotcounter{editnum}
	
	\ifx\distribution\True
		\newcommand{\edit}[1]{} 
	\else
		\newcommand{\edit}[1]{\stepcounter{editnum}\textbf{\textcolor{red}{$\blacksquare$}}} 
	\fi

\begin{document}
\title{Agent Spaces}
\author{John C. Raisbeck\footnote{Corresponding author, jraisbeck@mgh.harvard.edu}}
\author{Matthew W. Allen}
\author{Hakho Lee}
\affil{Center for Systems Biology\\ Massachusetts General Hospital\\ Boston, Massachusetts}
\maketitle


\begin{abstract}
	Exploration is one of the most important tasks in Reinforcement Learning, but it is not well-defined beyond finite problems in the Dynamic Programming paradigm (see \autoref{incompleteInformation}). We provide a reinterpretation of exploration which can be applied to any online learning method.
	
	We come to this definition by approaching exploration from a new direction. After finding that concepts of exploration created to solve simple Markov decision processes with Dynamic Programming are no longer broadly applicable, we reexamine exploration. Instead of extending the \emph{ends} of dynamic exploration procedures, we extend their \emph{means}. That is, rather than repeatedly sampling every state-action pair possible in a process, we define the act of modifying an agent to \emph{itself} be explorative. The resulting definition of exploration can be applied in infinite problems and non-dynamic learning methods, which the dynamic notion of exploration cannot tolerate.
	
	To understand the way that modifications of an agent affect learning, we describe a novel structure on the set of agents: a collection of distances (see \autoref{notAMetric}) $d_{\agentdefaut}, \agentdefaut \in \agentSet$, which represent the perspectives of each agent possible in the process. Using these distances, we define a topology and show that many important structures in Reinforcement Learning are well behaved under the topology induced by convergence in the agent space.
\end{abstract}
\newpage

\section{Introduction}\label{introduction} 
	Reinforcement Learning (RL) is the study of stochastic processes which are composed of a decision process~$\mathcal{P}$, and an agent~$\agentdefaut$.\edit{I don't think that this description is controversial, but I cannot find a citation which matches my description; in general, the ``tuple'' variation is used. I refuse to stoop to so low a level, so I'm a bit stuck.} The Reinforcement Learning problem of a decision process $\mathcal{P}$ with reward $R$ is to find an optimal agent $\agentdefaut^{*}$ which maximizes the expectation of a reward function $R : \Phi \rightarrow \mathbb{R}$ with respect to the distribution of paths $\Phi^{a}$ drawn from the process $(\mathcal{P}, \agentdefaut)$.\edit{As with the above, I cannot find a good external citation for this framework, even though it is entirely compatible with the literature.} Typically, Reinforcement Learning methods seek $\agentdefaut^{*}$ or otherwise high-reward agents via iterative learning processes, sometimes described as \emph{trial-and-error} \cite{Sutton2015}.
	
	In an online learning algorithm, learners send one or more agents to interact with the process and collect information on those interactions. After studying these interactions, the learner develops new agents to experiment with, seeking to improve the reward of successive generations of agents (see \autoref{defLearningAlgorithm}). In this pursuit, there is a trade-off between seeking high-quality agents during learning (exploitation) and seeking new information about the decision process (exploration).

	Exploitation, being closely related to general iterative optimization algorithms, is well-understood. In simple problems, its trade-off with exploration has been extensively researched \cite{Sutton2015, Gittins1979}. However, in more complex problems the conclusions reached by studying simple problems seem to have little bearing; some of the exploration methods which are least efficient in simple problems have been used in the most impressive demonstrations of the power of Reinforcement Learning to date \cite{Vinyals2019, OpenAI2019}.

	This paper focuses on exploration, especially on methods of exploration which ignore a certain set of the tenets of \emph{Dynamic Programming} \cite{Bellman1954, Bellman1957a}. We call these methods \emph{naïve}, and class among them \emph{Novelty Search} \cite{Lehman2011}, which we discuss in \autoref{novelty}. We begin in \autoref{reinforcementLearning} and \autoref{exploitation}, rigorously describing the problem of Reinforcement Learning and introducing the necessity for exploration. In \autoref{incompleteInformation} and \autoref{exploration}, we study the notions of exploration coming from the study of Dynamic Programming, and consider their efficacy in modern Reinforcement Learning. In \autoref{naiveExploration}, we investigate the properties of a class of spaces based on what we call \emph{primitive behavior}, and in \autoref{pseudometrics} introduce a general form of these spaces, which we call \emph{agent spaces}. The agent space can be defined in a broad class of decision processes and efficiently describes several important features of an agent. In \autoref{theAgentSpace}, it is demonstrated that the distributions of truncated paths $\phi_{t}$ are continuous in the agent, and in \autoref{continuity} it is demonstrated that certain functions of those paths (chiefly, expected reward) are continuous in the agent space.
	
\section{Reinforcement Learning}\label{reinforcementLearning}
	\subsection{Definitions}
	\begin{definition}[Decision Process]\label{defDecisionProcess}
		A discrete time \emph{decision process} $\mathcal{P}$ is a controlled stochastic process indexed by $\mathbb{N}$. Associated with the process are a set of states $S$, a set of actions $\actionSet$, and a state-transition function $\sigma$.
	\end{definition}
	\begin{definition}[State]
		A \emph{state} $s \in S$ is state possible in the decision process $\mathcal{P}$.
	\end{definition}
	\begin{definition}[Action]
		An \emph{action} $\action \in \actionSet$ is a control which an agent can exert on $\mathcal{P}$.
	\end{definition}
	\begin{definition}[Path]
		A \emph{path} (sometimes called a trajectory) $\phi \in \Phi$ in a decision process $\mathcal{P}$ is a sequence of state-action pairs generated by the interaction between the process and an agent $\agentdefaut$:
		\begin{align}
		\phi = (s_{0}, \action_{0}), (s_{1}, \action_{1}), (s_{2}, \action_{2})...
		\end{align}
	\end{definition}
	\begin{remark}
		We will sometimes need to refer to \emph{truncated paths} ($\phi_{t}$), i.e. paths which contain only the state-action pairs associated with indices $i \leq t$. This is common when referring to the domains of agents $a$ and the state transition function $\sigma$. We will see that truncated paths suffice to describe the domain of $\sigma$, but because agents act only when a new state has been generated, its domain is the set of truncated \emph{prime} paths $(\phi_{t}')$: paths which contain the $t$-th state, but not the $t$-th action.
		\begin{enumerate}
			\item $\phi_{t}=\{(s_{0}, \action_{0}), (s_{1}, \action_{1}), \ldots , (s_{t-1}, \action_{t-1}), (s_{t}, \action_{t})\}$\hfill\text{see }\autoref{paths}
			\item $\phi_{t}' = \{(s_{0}, \action_{0}), (s_{1}, \action_{1}), \ldots, (s_{t-1}, \action_{t-1}), s_{t}\}$\hfill\text{see }\autoref{primePaths}
		\end{enumerate}
		Sometimes (as in equation \eqref{sumReward}), it is necessary to refer to the $t$-th state or action of a path $\phi$. We denote the $t$-th state $\phi_{t}(s)$, and the $t$-th action $\phi_{t}(\action)$, using $s_{t}$ and $\action_{t}$ when unambiguous.
	\end{remark}
	\noindent Because $s_{0}$ has no antecedent, it is determined by the initial state distribution $s_{0} \sim \sigma_{0}$.
	\begin{definition}[Initial State Distribution]
		The \emph{initial state distribution} $\sigma_{0}$ of a process $\mathcal{P}$ is the probability distribution over $S$ which determines the first state in a path, $s_{0} = \phi_{0}(s) \sim \sigma_{0}$.
	\end{definition}
	\begin{definition}[State-Transition Function]
		The \emph{state-transition function} $\sigma$ of a process $\mathcal{P}$ is the function which takes the truncated path $\phi_{t-1}$ to the distribution of states $\sigma(\phi_{t-1})$ from which the next state $s_{t} = \phi_{t}(s)$ is drawn.
	\end{definition}\noindent
	Thus, the state at $t$ is determined by 
	\begin{align}
		s_{t} \sim
		\begin{cases}
			\sigma_{0}, & \text{if } t = 0,\\
			\sigma(\phi_{t-1}), & \text{if } t > 0.
		\end{cases}
	\end{align}
	\begin{definition}[Agent]\label{defAgent}
		An \emph{agent}\footnote{Some authors call this function a \emph{policy}, referring to its \emph{embodiment} as an agent. See \autoref{agent}, \autoref{defAgentIdentity}, and \autoref{parameterizedAgents}.} $\agentdefaut \in \agentSet$ is a function from the set of truncated paths $\Phi^{\prime} = \bigcup_{t \in \mathbb{N}}\Phi_{t}^{\prime}$ (notice that $\Phi^{\prime}$ contains only truncated paths, and $\Phi$ contains only infinite paths) of a process into the set of actions $\actionSet$ of the process:
		\begin{align}
			\agentdefaut : \bigcup_{t \in \mathbb{N}}\Phi_{t}^{\prime} \rightarrow \actionSet.
		\end{align}
	\end{definition}
	\ifx\figureson\True
	\vspace{-5ex}
	\begin{figure}[h]
		\begin{align}
		&\sbox0{\begin{tabular}[b]{c}\toprule P \\ \midrule\end{tabular}}
		\begin{tabular}{@{}c@{}c@{}}
		\begin{tabular}[t]{ c *{2}{>{$}c<{$}}}
		\toprule 
		& \multicolumn{2}{c}{Process} \\
		\cmidrule{2-3}
		0 & \sigma_{0}       & \agentdefaut(\phi_{0}') \\
		1 & \sigma(\phi_{0}) & \agentdefaut(\phi_{1}') \\
		2 & \sigma(\phi_{1}) & \agentdefaut(\phi_{2}') \\
		3 & \sigma(\phi_{2}) & \agentdefaut(\phi_{3}') \\
		4 & \sigma(\phi_{3}) & \agentdefaut(\phi_{4}') \\
		5 & \sigma(\phi_{4}) & \agentdefaut(\phi_{5}') \\
		$\vdots$ & \vdots & \vdots \\
		\end{tabular} &
		\rb{\phi_{0}'}{\phi_{0}}
		\rb{\phi_{0}' \\ \phi_{1}'}{\phi_{1}}
		\rb{\phi_{0}' \\ \phi_{1}' \\ \phi_{2}'}{\phi_{2}}
		\rb{\phi_{0}' \\ \phi_{1}' \\ \phi_{2}' \\ \phi_{3}'}{\phi_{3}}
		\rb{\phi_{0}' \\ \phi_{1}' \\ \phi_{2}' \\ \phi_{3}' \\ \phi_{4}'}{\phi_{4}}
		\rb{\phi_{0}' \\ \phi_{1}' \\ \phi_{2}' \\ \phi_{3}' \\ \phi_{4}' \\ \phi_{5}' }{\phi_{5}}
		\rb{\phi_{0}' \\ \phi_{1}' \\ \phi_{2}' \\ \phi_{3}' \\ \phi_{4}' \\ \phi_{5}' \\ x}{\phi}
		\end{tabular}
		\end{align}
		\vspace{-4ex}
		\caption{Truncated paths $\phi_{t}$ are drawn from a distribution defined by iteratively drawing states and new actions from $\sigma$ and $\agentdefaut$. The complete path $\phi$ is defined by the collection of these.}
		\label{paths}
	\end{figure}
	\fi
	\ifx\figureson\True
	\begin{figure}[!htb]
	\begin{align}
	&\sbox0{\begin{tabular}[b]{c}\toprule P \\ \midrule\end{tabular}}
	\begin{tabular}{@{}c@{}c@{}}
	\begin{tabular}[t]{ c *{2}{>{$}c<{$}}}
	\toprule
	& \multicolumn{2}{c}{Process} \\
	\cmidrule{2-3}
	0 & s_{0}       & 
	\end{tabular} &
	\rb{\phi_{0}'}{\phi_{0}' \rightarrow \action_{0} \sim a(\phi_{0}')}\vspace{-1ex}
	\end{tabular}\\
	&\sbox0{\begin{tabular}[b]{c}\toprule P \\ \midrule\end{tabular}}
	\begin{tabular}{@{}c@{}c@{}}
	\begin{tabular}[t]{ c *{2}{>{$}c<{$}}}
	\toprule
	& \multicolumn{2}{c}{Process} \\
	\cmidrule{2-3}
	0 & s_{0}       & \action_{0} \\
	\end{tabular} &
	\rb{\phi_{0}'}{\phi_{0}\rightarrow s_{1} \sim \sigma(\phi_{0})}\vspace{-1ex}
	\end{tabular}\\
	&\sbox0{\begin{tabular}[b]{c}\toprule P \\ \midrule\end{tabular}}
	\begin{tabular}{@{}c@{}c@{}}
	\begin{tabular}[t]{ c *{2}{>{$}c<{$}}}
	\toprule
	& \multicolumn{2}{c}{Process} \\
	\cmidrule{2-3}
	0 & s_{0}       & \action_{0} \\
	1 & s_{1} &  \\
	\end{tabular} &
	\rb{\phi_{0}'}{\phi_{0}}
	\rb{\phi_{0}' \\ \phi_{1}'}{\phi_{1}' \rightarrow \action_{1} \sim a(\phi_{1}')}\vspace{-1ex}
	\end{tabular}\\
	&\sbox0{\begin{tabular}[b]{c}\toprule P \\ \midrule\end{tabular}}
	\begin{tabular}{@{}c@{}c@{}}
	\begin{tabular}[t]{ c *{2}{>{$}c<{$}}}
	\toprule
	& \multicolumn{2}{c}{Process} \\
	\cmidrule{2-3}
	0 & s_{0}       & \action_{0} \\
	1 & s_{1} & \action_{1} \\
	\end{tabular} &
	\rb{\phi_{0}'}{\phi_{0}}
	\rb{\phi_{0}' \\ \phi_{1}'}{\phi_{1} \rightarrow s_{2} \sim \sigma(\phi_{1})}\vspace{-1ex}
	\end{tabular}\\
	&\sbox0{\begin{tabular}[b]{c}\toprule P \\ \midrule\end{tabular}}
	\begin{tabular}{@{}c@{}c@{}}
	\begin{tabular}[t]{ c *{2}{>{$}c<{$}}}
	\toprule
	& \multicolumn{2}{c}{Process} \\
	\cmidrule{2-3}
	0 & s_{0}       & \action_{0} \\
	1 & s_{1} & \action_{1} \\
	2 & s_{2} &  \\
	\end{tabular} &
	\rb{\phi_{0}'}{\phi_{0}}
	\rb{\phi_{0}' \\ \phi_{1}'}{\phi_{1}}
	\rb{\phi_{0}' \\ \phi_{1}' \\ \phi_{2}'}{\phi_{2}' \rightarrow \action_{2} \sim a(\phi_{2}')}\vspace{-1ex}
	\end{tabular}\\
	&\sbox0{\begin{tabular}[b]{c}\toprule P \\ \midrule\end{tabular}}
	\begin{tabular}{@{}c@{}c@{}}
	\begin{tabular}[t]{ c *{2}{>{$}c<{$}}}
	\toprule
	& \multicolumn{2}{c}{Process} \\
	\cmidrule{2-3}
	0 & s_{0}       & \action_{0} \\
	1 & s_{1} & \action_{1} \\
	2 & s_{2} & \action_{2}\\
	\end{tabular} &
	\rb{\phi_{0}'}{\phi_{0}}
	\rb{\phi_{0}' \\ \phi_{1}'}{\phi_{1}}
	\rb{\phi_{0}' \\ \phi_{1}' \\ \phi_{2}'}{\phi_{2} \rightarrow s_{3} \sim \sigma(\phi_{2})}\vspace{-1ex}
	\end{tabular}\\
	&\sbox0{\begin{tabular}[b]{c}\toprule P \\ \midrule\end{tabular}}
	\begin{tabular}{@{}c@{}c@{}}
	\begin{tabular}[t]{ c *{2}{>{$}c<{$}}}
	\toprule
	& \multicolumn{2}{c}{Process} \\
	\cmidrule{2-3}
	0 & s_{0}       & \action_{0} \\
	1 & s_{1} & \action_{1} \\
	2 & s_{2} & \action_{2}\\
	3 & s_{3} & \\
	\end{tabular} &
	\rb{\phi_{0}'}{\phi_{0}}
	\rb{\phi_{0}' \\ \phi_{1}'}{\phi_{1}}
	\rb{\phi_{0}' \\ \phi_{1}' \\ \phi_{2}'}{\phi_{2}}
	\rb{\phi_{0}' \\ \phi_{1}' \\ \phi_{2}' \\ \phi_{3}'}{\phi_{3}' \rightarrow \action_{3} \sim a(\phi_{3}')}
	\end{tabular}
	\end{align}
	\vspace{-1ex}
	\begin{displaymath}
	\vdots\hspace{\vdotsSpace}\vdots\hspace{\vdotsSpace}\vdots\hspace{\vdotsSpace}\vdots\hspace{\vdotsSpace}\vdots\hspace{\vdotsSpace}\vdots\hspace{\vdotsSpace}\vdots\hspace{\vdotsSpace}\vdots
	\end{displaymath}
		\vspace{-4ex}
		\caption{The structure of a decision process.}
		\label{primePaths}
		\vspace{-2ex}
	\end{figure}
	\fi

	Agents are judged by the quality of the control which they exert, as measured by the expectation of the reward of their paths.
	\begin{definition}[Reward]\label{defReward}
		A \emph{reward} function is a function from the set of paths $\Phi$ of a process $\mathcal{P}$ into the real numbers $\mathbb{R}$:
		\begin{align}
			R : \Phi \rightarrow \mathbb{R}\label{defRewardEquation}
		\end{align}
		Often \cite{Sutton2015}, $R$ can be described as a sum:
		\begin{align}
			R(\phi) =& \sum_{t \in \mathbb{N}} r(\phi_{t}(s), \phi_{t}(\action))\text{,}\hspace{0.75em}\text{or}\label{sumReward}\\
			R(\phi) =& \sum_{t \in \mathbb{N}} r(\phi_{t}(s), \phi_{t}(\action))\omega(t),
		\end{align}
		where $r : S \times \actionSet \rightarrow \mathbb{R}$ is an immediate \cite[49]{Watkins1989} reward function, and $\omega:\mathbb{N}\rightarrow[0,1]$ is a discount function such that
		\begin{align}
			\Omega = \sum_{t \in \mathbb{N}}\omega(t) < \infty.
		\end{align}
	\end{definition}
	
	\begin{definition}[Expected Reward]
		The \emph{expected reward} $J$ is the expectation of the reward of paths $\phi$ drawn from the distribution of paths $\Phi^{\agentdefaut}$ generated by the process $(\mathcal{P}, \agentdefaut)$:
		\begin{align}
			J(\agentdefaut) = \expectation_{\phi \sim \Phi^{\agentdefaut}}\left[R(\phi)\right]\label{defExpectedRewardEquation}
		\end{align}
	\end{definition}

	Reinforcement Learning is concerned with the \emph{optimal control} of decision processes, pursuing an optimal agent $\agentdefaut^{*}$ which achieves the greatest possible expected reward $J(\agentdefaut)$. To this end, \emph{learning algorithms} (see \autoref{learningAlgorithms}) are employed. Many algorithms pursue such agents by searching the agents which are ``near'' the agents which they have recently considered. In some contexts, this suffices to guarantee a solution in finitely many steps (see \autoref{PIT}). When this cannot be guaranteed, the problem is often restated in terms of discovering ``satisfactory'' agents, or discovering the highest quality agent possible under certain constraints.
	
	\begin{remark}[Table of Notation]
		We will frequently reference states, actions, paths, and sets and distributions of the same:\footnote{Without superscripts, $\Phi$ represents the set of possible paths. Note that only truncated paths can be prime.}
		
		{\centering\vspace{1.5ex}
			\renewcommand{\arraystretch}{1.5}
			\begin{tabular}{c|c|c|c}
				& Individual & \hspace{1.75em}Set\hspace{1.75em} & Distribution\\
				\hline
				State & $s$ & $S$ & $\sigma(\phi_{t})$\\\hline
				Action & $\action$ & $\actionSet$ & $\agentdefaut(\phi_{t}')$\\\hline
				Path & $\phi$ & $\Phi$ & $\Phi^{\agentdefaut}$\\\hline
				Prime Path & $\phi_{t}^{\prime}$ & $\Phi_{t}^{\prime}$ & $\Phi_{t}^{\prime \agentdefaut}$
			\end{tabular} 
			\vspace{1.5ex}
		\par}
	\end{remark}
	
	Often, the decision processes in Reinforcement Learning are assumed to be \emph{Markov} Decision Processes (MDPs); decision processes which satisfy the Markov property.
	\begin{definition}[Markov Property]\label{defMarkovProperty}
		A decision process has the \emph{Markov Property} if, given the prior state-action pair $(s_{t-1}, \action_{t-1})$, the distribution of the state-action pair $(s_{t}, \action_{t})$ is independent of the pairs $(s,\action)\in\{(s_{i},\action_{i})\,|\,i<t-1\}$.
	\end{definition}
	\begin{remark}[A Simple Condition for the Markov Property]\label{strictlyMarkov}
		Sometimes a slightly stricter definition of the Markov property is used: a decision process is guaranteed to be Markov if $\sigma$ is a function of $(s_{t}, \action_{t})$ and $\agentdefaut$ is a function of $s_{t+1}$ alone (which is itself a function of $(s_{t}, \action_{t})$). We call such agents \emph{strictly} Markov:
	\end{remark}
	\begin{definition}[Strictly Markov Agent]
		A \emph{strictly Markov agent} $\agentdefaut \in \agentSet$ is a function from the set of states $S$ into the set of actions $\actionSet$:
		\begin{align}
		\agentdefaut : S \rightarrow \actionSet.
		\end{align}
	\end{definition}

	\subsection{Computational considerations}
		When studying modern Reinforcement Learning, it is important to keep the constraints of practice in mind. Among the most important of these is the quantification of states and actions. In practice, the sets of states and actions are typically real (${S\subset\mathbb{R}^{m}, \actionSet\subset\mathbb{R}^{n}}$), finite ($|S|\in\mathbb{N}, |\actionSet|\in\mathbb{N}$), or a combination thereof.
		
		\subsubsection{The Agent}\label{agent}
			Agents can be represented in a variety of ways, but most modern algorithms represent agents with real function approximators (often neural networks) parameterized by a set of real numbers $\theta$,
			\begin{align}
				\agentdefaut_{\theta} : \mathbb{R}^{m} \rightarrow \mathbb{R}^{n}.
			\end{align}
			This is true even in cases where the sets of states or actions are not real. To reconcile an approximator's domain and range with such sets, the approximator is fit with an \emph{input}~function~$I$ and an \emph{output}~function~$O$, which mediate the interaction between a process and a function~approximator:
			\begin{align}
				I &: \bigcup_{t \in \mathbb{N}}\Phi_{t}^{\prime} \rightarrow \mathbb{R}^{m},\\
				O &: \mathbb{R}^{n} \rightarrow \actionSet.
			\end{align}
			When composed with these functions, the approximator forms an agent:
		 	\begin{align}
		 		\agentdefaut = O \circ \agentdefaut_{\theta} \circ I.
			\end{align}
	
		\subsubsection{The Output Function}\label{outputFunction}
			The choice of output function is important and requires consideration of several aspects of the process, including the learning process and the action set $\actionSet$. One of the most consequential roles of an output function $O$ occurs in processes where $|\actionSet|$ is finite. In such processes, the output function must map a real vector to a discrete action.
			
			Some learning algorithms, such as $\mathcal{Q}$-Learning, invoke the action-value function $\mathcal{Q}$ (see \autoref{defActionValueFunction}) in their learning process, training an agent to estimate the expected reward of each action at each state, given an agent $\agentdefaut$ \cite[81]{Bellman1962}. Because the object of approximation in $\mathcal{Q}$-learning is defined analytically, some output functions are less sensible than others. Other algorithms, like policy gradients, have outputs with less fixed meanings. Consequently, they can accommodate a variety of output functions, and the output function has a substantial effect on the agent itself. Let us consider some of the most common output functions. For further discussion of output functions, see \appref{undirected}.
			\begin{definition}[Greedy Action Sampling]\label{greedy}
				In Greedy Action Sampling, the function approximator's output is taken to indicate a single ``greedy'' action, and this greedy action is taken\cite{Sutton2015}. In problems with finite sets of actions, the range of the parameterized agent is real ($\agentdefaut_{\theta}(\phi_{t}^{\prime}) \in \mathbb{R}^{|\actionSet|}$) and the action associated with the dimension of greatest magnitude (the \emph{greedy} action) is taken.
			\end{definition}
			
			\begin{definition}[$\varepsilon$-Greedy Sampling]\label{defEpsilonGreedy}
				In $\varepsilon$-Greedy Sampling, $\varepsilon \in [0,1]$, the greedy action is taken with probability $1 - \varepsilon$. Otherwise, an action is drawn from the uniform distribution on $\actionSet$.\cite{Sutton2015}
			\end{definition}\noindent
			Unlike the greedy sampling methods above, Thompson sampling requires a finite set of actions.
			\begin{definition}[Thompson Sampling]
				In Thompson Sampling, the range of the parametrized agent $\agentdefaut_{\theta}$ is $\{x \in \mathbb{R}^{|\actionSet|}| \sum_{i\in |\actionSet|} x_{i} = 1, x_{i} > 0\}$, and the action is selected by drawing from the random variable which gives action $\action_{i}$ probability $x_{i}$ \cite{Agrawal2013}.
			\end{definition}
		
		\subsection{Learning Algorithms}\label{learningAlgorithms}
			We are concerned primarily with \emph{online} learning algorithms, in which exploration is most sensible. Whereas most dynamic programming based methods study only a single agent in each epoch of training, we consider a broader class of learning algorithms:
			\begin{definition}[Learning Algorithm]\label{defLearningAlgorithm}
				A \emph{learning algorithm} (sometimes \emph{optimizer}) is an algorithm which generates sets of candidate agents $\agentSet_{n}$, observes their interactions with a process, and uses this information to generate a  set of candidates $\agentSet_{n+1}$ in the next epoch. This procedure is repeated for each epoch $n \in I$ to improve the greatest expected reward among considered agents,
				\begin{align}
				\sup_{n\in I}\sup_{\agentdefaut\in \agentSet_{n}}J(\agentdefaut).\label{supsup}
				\end{align}
			\end{definition}
			\begin{remark}[Loci]
				Many learning algorithms have an additional property which can simplify discussions of learning algorithms: they center  generations of agents $\agentSet_{n}$ around a locus agent, or locus of optimization, denoted $\agentdefaut_{\text{locus}}$.
			\end{remark}
			Not every interesting algorithm falls within this definition (e.g. evolutionary methods with multiple loci), but many of the discussions in this paper apply, \emph{mutatis mutandis}, to a broader class of methods.
			
		\subsection{Interpreting the Reinforcement Learning Problem}\label{aNote} 
			Many learning algorithms are influenced by a philosophy of optimization called \emph{Dynamic Programming}~\cite{Bellman1954}~(DP). Dynamic programming approaches control problems state by state, determining the best action for each state according to an approximation of the \emph{state-action value function} $\mathcal{Q}^{\agentdefaut}(s,\action)$ (see \autoref{defActionValueFunction}). DP has compelling guarantees in finite\footnote{An MDP is called \emph{finite} if $|\actionSet| < \aleph_{0}$ and $|S| < \aleph_{0}$.} strictly Markov decision processes \cite{Watkins1992}. However, because DP methods require every \cite{Watkins1989, Howard1960} action value be approximated during each step of optimization, these guarantees cannot transfer to infinite problems.
			
			Although DP is dominant in Reinforcement Learning \cite[73]{Sutton2015}, several trends indicate that it may not be required for effective learning. Broadly, the problems of interest in Reinforcement Learning have shifted substantially from those Bellman considered in the early 1950s \cite{Bellman1962}. Among other changes, the field now studies many infinite problems, agents are not generally tabular, and the problems of interest are not in general Markov. Each of these changes independently prevents a problem from satisfying the requirements of dynamic programming \cite{Howard1960}. Simultaneously, approaches not based on DP, like Evolution Strategies \cite{Salimans2017} have shown that DP methods are not necessary to achieve performance on par with modern dynamic programming based methods in common reinforcement learning problems. Taken together, these trends indicate that DP may not be uniquely effective.
			
			This paper considers exploration in the light of these developments. We distinguish DP methods like $\mathcal{Q}$-Learning, which treat reinforcement learning problems as collections of subproblems, one for each state, from methods like Evolution Strategies \cite{Salimans2017}, which treat reinforcement learning problems as unitary. We call the former class \emph{dynamic}, and the latter class \emph{naïve}.
		
\section{Exploitation and its discontents}\label{exploitation}
	The study of Reinforcement Learning employs several abstractions to describe and compare the dynamics of learning algorithms which may have little in common. One particularly important concept is a division between two supposed tasks of learning: \emph{exploitation} and \emph{exploration}. Exploitation designates behavior which targets greater reward in the short term and exploration designates behavior which seeks to improve the learner's understanding of the process. Frequently, these tasks conflict; learning more (exploration) can exclude experimenting with agents which are estimated to obtain high reward (exploitation) and vice versa. The problem of balancing these tasks is known as the \emph{Exploration versus Exploitation} problem. Let us consider exploitation, exploration, and their conflicts.
	
	\subsection{Exploitation}
		Exploitation can be defined in a version of the Reinforcement Learning problem in which the goal is to maximize the \emph{cumulative} total of rewards obtained by all sampled agents: let $\phi^{\agentdefaut}$ be a path sampled from $(\mathcal{P}, \agentdefaut)$
		\begin{align}
			\sum_{\agentdefaut \in \left(\bigcup_{n \in I} \agentSet_{n}\right)}R(\phi^{\agentdefaut}).
		\end{align}
		\begin{definition}[Exploitation]
			A learning method is exploitative if its purpose is to increase the reward accumulated by the agents considered in the epochs $I$.
		\end{definition}
		
		Fortunately, this definition of exploitation remains relevant to our problem because it informs the cumulative version of the problem. Exploitation advises the dispatch of agents which are expected to be high-quality, rather than those whose behavior and quality are less certain. In effect, exploitation is conservative, preferring ``safer'' experimentation and incremental improvement to potentially destructive exploration. While exploitation is crucial to reinforcement learning, explorative experiments are often necessary because \emph{sometimes, exploitation alone fails}.
	
	\subsection{Exploitation Fails}\label{exFails}
		When exploitation fails, it is because its conservatism causes it to become stuck in local optima. Because Reinforcement Learning methods change the locus agent gradually, and exploitative methods typically generate sets of agents $A_{n}$ near the locus, exploitative learners can become trapped in local optima. \autoref{unid} provides a unidimensional representation of this problem: because the agents $a \in A_{n}$ fall within the sampling radius, and updates to the locus agent are generally interpolative, the product of the learning process depends exclusively on the initial locus and the sampling distance.
		
		\ifx\figureson\True
		\vspace{1ex}
		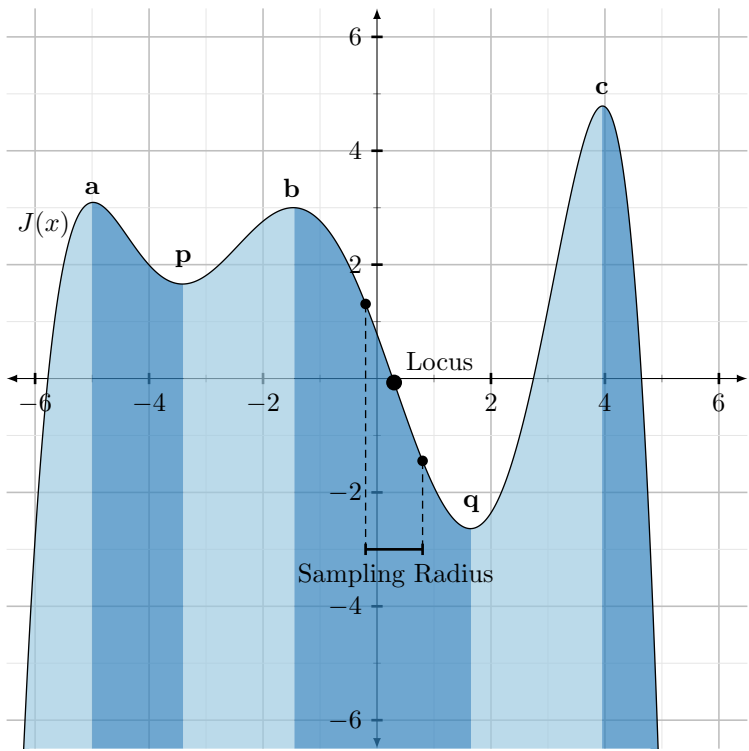
\begin{figure}
			{\centering
				\def\FunctionF(#1){1/3*(-8/945*(#1)^6 - 181/10080*(#1)^5 + 9761/30240*(#1)^4 + 7001/10080*(#1)^3 - 16805/6048*(#1)^2 - 3095/504*(#1))}%
				\def\FunctionG(#1){\FunctionF(#1+0.5)+2}
				\begin{tikzpicture}
					\begin{axis}[
						restrict y to domain=-20:20,
						axis y line=center,
						axis x line=middle, 
						axis on top=false,
						xmin=-6.5,
						xmax=6.5,
						ymin=-6.5, 
						ymax=6.5, 
						height=\figuresize,
						width=\figuresize,
						grid=both,
						ytick={-6,-4,...,6},
						xtick={-6,-4,...,6},
						minor tick num=1,
						major tick style={color=black, line width=1pt},
						grid style={line width=.1pt, draw=black!10},
						major grid style={line width=0.6pt, color=black!25},
						axis line style={latex-latex} 
						]
						\addplot [name path=f, domain=-6.5:6, samples=\samples, mark=none, line width=0.5pt] {\FunctionG(x)};		

						\path[name path=axis] (axis cs:-6.1,-6.5) -- (axis cs:4.8,-6.5);
						\addplot [
						thick,
						color=\colorone,
						fill=\colorone, 
						fill opacity=0.7
						]
						fill between[
						of=f and axis,
						soft clip={domain=-6.5:-5},
						];

						\addplot [
						thick,
						color=\colortwo,
						fill=\colortwo, 
						fill opacity=0.7
						]
						fill between[
						of=f and axis,
						soft clip={domain=-5:-3.4},
						];

						\addplot [
						thick,
						color=\colorone,
						fill=\colorone, 
						fill opacity=0.7
						]
						fill between[
						of=f and axis,
						soft clip={domain=-3.4:-1.45},
						];
						
						\addplot [
						thick,
						color=\colortwo,
						fill=\colortwo, 
						fill opacity=0.7
						]
						fill between[
						of=f and axis,
						soft clip={domain=-1.45:1.65},
						];
		
						\addplot [
						thick,
						color=\colorone,
						fill=\colorone, 
						fill opacity=0.7
						]
						fill between[
						of=f and axis,
						soft clip={domain=1.65:3.95},
						];
						
						\addplot [
						thick,
						color=\colortwo,
						fill=\colortwo, 
						fill opacity=0.7
						]
						fill between[
						of=f and axis,
						soft clip={domain=3.95:6.5},
						];

						\path[name path=locus] (axis cs: 0.3,-6.5) -- (axis cs:0.3,6.5);
						\fill [name intersections={of=f and locus,by=E}, color=\colorthree] (E) circle[radius=3pt];
						
						\path[name path=left] (axis cs: -0.2,-6.5) -- (axis cs:-0.2,6.5);
						\addplot +[mark=none, line width=0.5pt, color=\colorthree] coordinates {(-0.2,1.3) (-0.2,-3)};
						\fill [name intersections={of=f and left,by=Eleft}, color=\colorthree] (Eleft) circle[radius=2pt];

						\path[name path=right] (axis cs: 0.8,-6.5) -- (axis cs:0.8,6.5);
						\addplot +[mark=none, line width=0.5pt, color=\colorthree] coordinates {(0.8,-1.45) (0.8,-3)};
						\fill [name intersections={of=f and right,by=Eright}, color=\colorthree] (Eright) circle[radius=2pt];
						
						\addplot [mark=none, line width=1pt, color=\colorthree] coordinates {(-0.2,-3) (0.8,-3)};
						\addplot [mark=none, line width=1pt, color=\colorthree] coordinates {(-0.2,-2.9) (-0.2,-3.1)};
						\addplot [mark=none, line width=1pt, color=\colorthree] coordinates {(0.8,-2.9) (0.8,-3.1)};

						
						\node[color=\colorfour] at (axis cs: -5.85, 2.7) {$J(x)$};
						\node[color=\colorfour] at (axis cs: -5, 3.35) {\textbf{a}};
						\node[color=\colorfour] at (axis cs: -1.5, 3.35) {\textbf{b}};
						\node[color=\colorfour] at (axis cs: 3.95, 5.1) {\textbf{c}};
						\node[color=\colorfour] at (axis cs: -3.4, 2.1) {\textbf{p}};
						\node[color=\colorfour] at (axis cs: 1.65, -2.2) {\textbf{q}};
						\node[color=\colorfour] at (axis cs: 1.1, 0.3) {Locus};
						\node[color=\colorfour] at (axis cs: 0.33, -3.45) {Sampling Radius};
					\end{axis}
				\end{tikzpicture}\\
				\par}
			\caption{A depiction of the reward of agents parameterized by a real number $x$. Even though \textbf{c} has higher expected reward, it is isolated from the locus. Within the sampling radius, the trend is clear: lower values of $x$ produce greater reward. Thus, a purely exploitative method will end at \textbf{b}.}
			\label{unid}
			\vspace{-2ex}
		\end{figure}
		\fi
		
		\subsubsection{The Policy Improvement Theorem}
		No discussion of local optima in Reinforcement Learning would be complete without a discussion of the Policy Improvement Theorem \cite{Howard1960, Watkins1989}. Let us begin with a definition:
		\begin{definition}[Action Value Function ($\mathcal{Q}$)]\label{defActionValueFunction}
			Let $\mathcal{P}$ be a strictly Markov decision process (\autoref{strictlyMarkov}). Then, the \emph{Action Value Function} is the expectation of reward\footnote{Reward is assumed to be summable \eqref{sumReward}, and $\omega(t)$ is assumed to be exponential, $\omega(t) = \gamma^{t}, 0 \leq \gamma < 1$.}, starting at state $s$, taking an action $\action$, and continuing with an agent $a$: \cite{Watkins1989}
			\begin{align}
				\mathcal{Q}^{\agentdefaut}(s, \action)=\expectation\left[R(\phi)\hspace{0.3em}|\hspace{0.3em}s_{0} = s,\hspace{1em} \action_{0} = \action, \hspace{1em} \Forall t \in \mathbb{Z}^{+} (\action_{t} \sim \agentdefaut(s_{t-1}))\right].
			\end{align}
		\end{definition}
		
		\begin{theorem}[Policy Improvement Theorem]\label{PIT}
			For any agent $\agentdefaut$ in a finite, discounted, bounded-reward strictly Markov decision process, either there is an agent $\agentdefaut'$ such that
			\begin{align}
				\Forall& s \in S \big(\mathcal{Q}^{\agentdefaut}(s,\agentdefaut'(s))\geq \mathcal{Q}^{\agentdefaut}(s,\agentdefaut(s))\big) \land \Exists s \in S\left(\mathcal{Q}^{\agentdefaut}(s,\agentdefaut'(s)) > \mathcal{Q}^{\agentdefaut}(s,\agentdefaut(s))\right)
			\end{align}
			or $\agentdefaut$ is an optimal agent.
		\end{theorem}
		This theorem also holds if this condition is appended:
		\begin{align}
			\Exists !& s \in S(\agentdefaut'(s) \neq \agentdefaut(s)).
		\end{align}
		With this condition, \autoref{PIT} shows that every agent $\agentdefaut$ either has a superior neighbor $\agentdefaut'$ (an agent which differs from $\agentdefaut$ in its response to exactly one state), or it is optimal.
		
		This means that any \emph{finite} strictly Markov decision process has a discrete ``convexity''\hspace{-0.1em}; every imperfect agent has a superior neighbor. Unfortunately, many of the problems of interest in modern Reinforcement Learning are not finite or, if finite, are too large for the \emph{Policy Improvement Algorithm}, which exploits this theorem, to be tractable \cite[54]{Watkins1989}. 
		
	\subsection{The Assumptions of Dynamic Programming}
		Modern Reinforcement Learning confronts many problems which do not satisfy the assumptions of the Policy Improvement Theorem. Many modern problems fail directly, using infinite decision processes. Many others \emph{technically} satisfy the requirements of the theorem, but are so large as to render the guarantees of policy improvement notional with modern computational techniques. Still others fail the qualitative conditions, for example, by not being strictly Markov.
	
		The calculations necessary to guarantee that one policy is an improvement upon another require that the learner have knowledge of each state and action in the process, as well as of the dynamics of the decision process. While excessive size of $S$ or $\actionSet$ are easiest to exhibit, the interplay of the size of these sets tends to be the true source of the problem. Difficulties caused by this interplay are known as the \emph{curse of dimensionality}, a term which Bellman used to describe the way that problems with many aspects can be more difficult than the sum of their parts \cite[ix]{Bellman1957a}. For example, the size of the set of agents grows exponentially in the number of states and polynomially in the number of actions:
		\begin{align}
		|\agentSet| = |\actionSet|^{|S|}.
		\end{align}
		In large problems, the curse of dimensionality makes the guarantees of Dynamic Programming almost impossible to achieve in practice. Curiously, the performance of DP methods seems to degrade slowly, even as the information necessary for the guarantees of DP quickly becomes unachievable.
		
	\subsection{Incomplete Information and Dynamic Programming}\label{incompleteInformation}
		One way to address the size of a Reinforcement Learning problem is to collect the information necessary for Dynamic Programming more efficiently. Dynamic Programming methods which approximate $\mathcal{Q}^{\agentdefaut}$ rely on several types of information, all based in the superiority condition:
		\begin{align}
			\Forall s \in S \big(\mathcal{Q}^{\agentdefaut}(s,\agentdefaut'(s))\geq \mathcal{Q}^{\agentdefaut}(s,\agentdefaut(s))\big).
		\end{align}
		 $\mathcal{Q}$-based methods rely on the learner's ability to approximate the $\mathcal{Q}^{\agentdefaut}$-function. There are two ways to approach this problem: $\mathcal{Q}^{\agentdefaut}$ can be approximated directly (i.e. separately for each state), or it can be calculated from knowledge or approximations of several aspects of the process:
		\begin{enumerate}
			\item The set of states, $S$,
			\item The set of actions, $\actionSet$,
			\item The immediate reward function, $r$,
			\item The state-transition function, $\sigma$, and
			\item The agent $\agentdefaut$.
		\end{enumerate}
		
		Calculating the $\mathcal{Q}^{\agentdefaut}$ function is, in a sense, more efficient: if the quantities above are known to an acceptable degree of precision, then a consistent action-value function can be imputed without repeated sampling.
		
		The complexity of this operation is a function of $|S| \times |\actionSet|$. In a sense, this method is inexpensive; the number of state-action pairs is much smaller than the number of possible sample paths, $|\Phi_{t}| = |S \times \actionSet|^{t} > |S \times \actionSet|$. It is also typically much smaller than the number of agents, $|\actionSet|^{S}$.
		
		This kind of approximation is fundamental to Dynamic Programming methods, and serves as the basis for many exploration methods. In general, reinforcement learning algorithms are assumed to have full knowledge of $\actionSet$ and $\agentdefaut$, but they may not have complete information about $S$, $r$, or $\sigma$. Thus, exploration has sometimes been defined as pursuing experiences of state-action pairs $(s, \action)$, specifically the quadruplets which can be associated with them in a path, $(s_{t}, \action, r(s_{t}, \action), s_{t+1})$. If enough of these are collected, it is possible to approximate the expected reward of each state-action pair, as well as the transition probabilities $\sigma(s,\action)$.
		
		Importantly, because these quadruplets can be induced by particular actions (i.e., $(s, \action)$ can be generated by  taking the action $\action$ at a time when $s$ is visited), the task of collecting the information about these quadruplets can be reduced to a simple two-step formula: first, visit each state $s \in S$, and second, attempt each action in that state. Under the right circumstances, proceeding in this way results in experience of every state-action pair, $\{(s,\action), | s \in S, \action \in \actionSet\}$.
		
		This simple notion of exploration is both efficient and effective in some circumstances: if $S$ is small and all of its states may be easily visited, and $\actionSet$ is small, then it is easy to consider every possible state-action pair. After enough sampling, this allows the $\mathcal{Q}$-function to be approximated. Outside of problems which satisfy those conditions, it is natural to consider other notions of exploration. Sometimes, these definitions take a more descriptive form, for example in Thrun: ``exploration seeks to minimize learning time'' \cite{Thrun1992}. Even without sufficient information about the process to \emph{guarantee} policy improvement, Dynamic Programming methods have performed admirably \cite{Vinyals2019, OpenAI2019}. 
		
		These results could be seen as a testament to the efficacy of Dynamic Programming under non-ideal circumstances, however, there is a curious countervailing trend: in many of these problems, simple or black-box methods such as Evolution Strategies \cite{Salimans2017} (an implementation of finite-differences gradient approximation) have been able to match the performance of modern Dynamic Programming methods. This conflicts with the present theory in two ways: first, it challenges the idea that Dynamic Programming is uniquely efficient, or uniquely suited to Reinforcement Learning problems. Second, because these methods are not dynamic, they lack the usual information requirements which exploration is supposed to resolve, yet they are useless without exploration (as seen in \autoref{unid}) - if information about state-action pairs is not directly employed, what is the role of exploration?
		
		In the next section we begin by describing the properties of the methods of \autoref{outputFunction} which guarantee that \emph{eventually}, every state-action pair will be experienced. We then discuss methods which address circumstances where certain states are difficult to handle
	
\section{Exploration and contentment}\label{exploration}
	Dynamic exploration can be divided into two categories: directed exploration and undirected exploration \cite{Thrun1992}. Undirected methods explore by using random output functions like $\varepsilon$-Greedy Sampling and Thompson Sampling (see \autoref{outputFunction}) to experience paths which would be impossible with the corresponding greedy agent. These are called \emph{un}directed because the changes which the output functions make to the greedy version of the agent are not intended to cause the agent to visit particular states. Instead, these methods explore through a sort of serendipity.
	
	Directed methods have specific goals and mechanisms more narrowly tailored to the optimization paradigm they support. Some directed methods, like \#Exploration \cite{Tang2017} (see \autoref{hashtagExploration}) seek to experience particular state-action pairs by directing the learner through \emph{exploration bonuses} to consider agents which lead to state-action pairs which have been visited less in the learning history. In order to apply the state-action pair formulation of exploration to large and even infinite problems, \#Exploration employs a hashing function to simplify and discretize the set set of states.
 
	Other directed exploration methods, like that of Stadie et al. \cite{Stadie2015} (see \autoref{stadieDefinition}) employ exploration bonuses to incentivize agents which visit states which are poorly \emph{understood}. Stadie et al. begin by modeling the state-transition function (in a strictly Markov process) $\sigma$ with a function $\mathcal{M}$. In each step of the process, the model $\mathcal{M}(s_{t}, \action_{t})$ guesses the next state $s_{t+1}$. After guessing, the model is trained on the transition from $(s_{t},\action_{t})$ to $s_{t+1}$. When the model is less accurate (i.e. when the \emph{distance}\footnote{Stadie et al. assume that $S$ is a metric space.} between $s_{t+1}$ and $\mathcal{M}(s_{t}, \action_{t})$ is large), Stadie et al. reason, there is more to be learned about $s_{t}$, and their method assigns exploration bonuses to encourage agents which visit $s_{t}$. Unfortunately, none of these methods can recover the guarantees of Dynamic Programming in problems which do not satisfy the requirements of the theory of Dynamic Programming. More detailed descriptions of these methods may be found in \autoref{explorationMethods}.
	
	\subsection{The Essence of Exploration}
		The brief survey above and in \autoref{explorationMethods} is far from complete, but it contains the core strains of most modern exploration methods. A more thorough discussion of exploration may be found in \cite{Thrun1992} or \cite{Weng2020}. In spite of its incompleteness, our survey suffices for us to reason generally about exploration, and about its greatest mystery: why do methods which were developed to collect exactly the information necessary for dynamic programming appear to help naïve methods succeed?
		
		Because naïve methods do not make use of action-values, nor do they collect information about particular state-action pairs, the dynamic motivations of these exploration methods cannot explain their efficacy when paired with naïve methods of exploitation. Instead, there must be something about the process of exploration itself which aids naïve methods. That poses a further challenge to the dynamic paradigm: if exploration is effective in naïve methods for non-dynamic reasons, to what extent do those reasons contribute to their effect in dynamic methods?
	
		These exploration methods seem to share little beyond their motivations. One other thing which they share - and which they by necessity share with \emph{every} reinforcement learning algorithm - is that their mechanism is, ultimately, aiding in the selection of the next set of agents $\agentSet_{n}$. Undirected methods accomplish this by selecting an output function, and directed methods go slightly further in influencing the parameters of the agents, but this is their shared fundamental mechanism.
		
		What differentiates exploration methods from other reinforcement learning methods is that they influence the selection of agents not to improve reward in the next epoch, as is standard in exploitation methods, but to collect a more diverse range of information about the process. It is the combination of this mechanism and purpose which makes a method explorative:
		\begin{definition}[Exploration]
			A reinforcement learning method is \emph{explorative} if it influences the agents $\agentdefaut \in \agentSet$ for the purpose of information collection.
		\end{definition}
		We now know two things about exploration: in \autoref{exploitation} we established that exploration was a process which sought additional information about the process. We have now added that exploration is accomplished \emph{by} changing the agents which the learner considers. This, however, does not resolve our question: under the dynamic programming paradigm, these changes are made so as to collect the information necessary to calculate action-values. What is the information which naïve methods require, and how does dynamic exploration collect it? To what extent does that other information contribute to the effectiveness of those methods in dynamic programming?
		
		To continue our study of exploration in naïve learning, we begin in the next section with a discussion of \emph{Novelty Search}, an algorithm which uses a practitioner-defined \emph{behavior function} $\behavior$ to explore \emph{behavior spaces}. In \autoref{pseudometrics}, we describe a general substrate for naïve exploration which is general enough to contain other exploration substrates, \emph{and} is equipped with a useful topological structure.
		
\section{Naïve Exploration}\label{naiveExploration}
	One of the most prominent examples of naïve exploration is an algorithm called \emph{Novelty Search} \cite{Lehman2011}. In contrast to the other methods which we discuss in this work, its creators do not describe it as a learning algorithm. Instead, they call novelty search an ``open-ended'' method. Nonetheless, methods which incorporate Novelty Search can usually be analyzed as learning algorithms, since they typically satisfy \autoref{defLearningAlgorithm}, with the possible exception of the ``purpose'' of the method. 
	
	\subsection{Novelty Search}\label{novelty}
		Novelty Search is an algorithmic component of many learning algorithms which was introduced by Joel Lehman and Kenneth O. Stanley \cite{Lehman2011, Lehman2008}. Unlike other learning methods, Novelty Search works to encourage \emph{novel}, rather than high-reward agents.
	
		\begin{definition}[Novelty Search]
			Novelty Search is a component which can be incorporated into many learning algorithms which defines the \emph{behavior} $\behavior(\agentdefaut)$ and \emph{novelty} $N(\agentdefaut)$ of agents which the learning algorithm considers $\agentdefaut \in \agentSet_{n}$. 
			
			Because Novelty Search does not specify an optimizer, the details of implementation can vary, but the ``search'' in Novelty Search refers to the way that Novelty Search methods seek agents with higher novelty scores $N(\agentdefaut)$. These scores are based upon the scarcity of an agent's behavior $\behavior(\agentdefaut)$ within a behavior archive $\chi$ \eqref{behaviorArchive}.
		\end{definition}
	
		\begin{definition}[Behavior]
			The \emph{behavior} of an agent is the image of that agent\footnote{In practice, many behavior functions are functions of a sampled path of the agent, rather than the agent itself.} under a \emph{behavior function} (or behavior \emph{characterization})
			\begin{align}
				\behavior:\agentSet\rightarrow X \label{behavior};
			\end{align}
			a function from the set of agents $\agentSet$ into a space of behaviors $X$ equipped with a distance $d_{X}$.\footnote{\label{notAMetric}
				The literature on Novelty Search is most sensible when the range of the behavior function is assumed to be a metric space, and novelty search is usually discussed under that pretense. However the ``novelty metrics'' employed in \emph{Abandoning Objectives} are \textbf{not} metrics in the mathematical sense (see \autoref{metricAxioms}) Instead, they are \emph{squared} Euclidean distances - a symmetric \cite{Arkhangelskii1990}. We use the word \emph{distance} to refer to a broader class of functions, and use the word \emph{metric} and its derivatives in the formal sense.
			}
		
			The behavior \emph{archive} $\chi$ in Novelty Search is a subset of the behaviors which have been observed in the learning process,
			\begin{align}
				\chi \subset \bigcup_{m \leq n} \{\behavior(\agentdefaut) \,|\, \agentdefaut \in \agentSet_{m}\}.\label{behaviorArchive}
			\end{align}
			In general, the archive is meant to summarize the behaviors which have been observed so far using as few representative behaviors as possible, so as to minimize computational requirements.
		\end{definition}
		
		\begin{definition}
			The novelty $N$ of a behavior $\behavior(\agentdefaut)$ is a measure of the sparsity of the behavior archive around $\behavior(\agentdefaut)$. In \emph{Abandoning Objectives}, Lehman and Stanley use the average distance\footref{notAMetric} from that behavior to its $k$ nearest neighbors in the archive, $K \subset \chi$:
			\begin{align}
				N(\behavior(\agentun)) = \frac{1}{k}\sum_{\behavior(\agenttrois) \in K} d_{X}(\behavior(\agentun), \behavior(\agenttrois)).
			\end{align}
		\end{definition}
		
		Novelty Search reveals something about naïve methods as a class: because they do not operate under the dynamic paradigm, individually manipulating the ways that agents respond to each situation possible in the process, they must employ another structure to understand the agents which it considers, and to determine $\agentSet_{n+1}$. For this purpose, Novelty Search uses the structure of the chosen behavior space X. Let us consider the structures other naïve methods might use.
		
		In the case of exploitation, a simple structure is available: reward. The purpose of exploitation is to improve reward, so reward is the relevant structure. Under the dynamic framework, reward is decomposed into the immediate rewards $r(s, \action)$, and agents are specified in relation to these. In the naïve framework, that decomposition is not availed, so only the coarser reward of an agent, $R(\agentdefaut)$, can be used. In some problems, other structures may correlate with reward, but these correlations can be inverted by changes to the state-transition function or reward function, so the only a priori justifiable structure for exploitation is reward.
		
		Exploration is more complex. In Dynamic Programming, the information necessary to solve a Reinforcement Learning problem is well defined, but in the naïve framework, there is not a general notion of information ``sufficient'' to solve a problem (except for exhaustion of the set of agents). That is, naïve exploration does not have a natural definition in the same way as naïve exploitation or dynamic exploration. Let us then consider a definition of exclusion:
		
		\begin{definition}[Naïve Exploration]\label{defNaiveExploration}
			A learning method is a method of naïve exploration if
			\begin{enumerate}[i.]
				\item The method itself is naïve, and
				\item The set of agents is explored using a structure that is not induced by the expected reward.
			\end{enumerate}
		\end{definition}
		
		Under this definition, Novelty Search is clearly a method of naïve exploration. Let us consider it further. Rather than treat it as a unique algorithm, we can consider Novelty Search to be a family of exploration methods, each characterized by the way that it projects the set of agents $\agentSet$ into a space of behaviors $X$.
		
		Novelty Search can be analyzed with respect to several goals: it may be viewed as an explorative method, or, when paired with a method of optimization, it may be seen as an open-ended or learning method in and of itself. Unfortunately, the capacity of Novelty Search to accomplish any of these goals is compromised by the subjectivity of the behavior function. Because the behavior function relies on human input, the exploration which is undertaken, the diversity which Novelty Search achieves, and the reward at the end of a learning process involving Novelty Search all rely on the beliefs of the practitioner. Instead of viewing this subjectivity as a problem, Lehman and Stanley embrace it, suggesting that behavior functions \emph{must} be determined manually for each problem, writing:
		\begin{quote}
			There are many potential ways to measure novelty by analyzing and quantifying behaviors to characterize their differences. Importantly, like the fitness function, this measure must be fitted to the domain.\hfill -Lehman and Stanley \cite{Lehman2011a}
		\end{quote}
		If followed, this advice would make it virtually impossible to disentangle Novelty Search as an algorithm from either the problems to which it is applied or the practitioners applying it. Fortunately, some authors have rejected this suggestion, pursuing more general notions of behavior.
		
		We now present a brief overview of some of the behavior functions in the literature, including those described in Lehman and Stanley's pioneering Novelty Search papers \cite{Lehman2011, Lehman2008}. \appref{behaviorFunctions} presents a more detailed discussion of these functions as well as some other behavior functions which could not be included in this summary.
		
		In their flagship paper on Novelty Search, Lehman and Stanley \cite{Lehman2011} consider as their primary example problem a two-dimension maze. As a secondary example they take a similar navigation problem in three dimensions. Importantly for our discussion, behavior functions in both environments admit a concept of the agent's \emph{position}. Lehman and Stanley introduce two behavior functions in their study of Novelty Search, both of which are functions of the position of the agent throughout a sampled truncated path $\phi_{t}^{\agentdefaut} \sim \Phi^{\agentdefaut}$.
		
		The simpler of these behavior function takes the agent's final position (i.e. the position of the agent when the path is truncated), and the more complex behavior function takes as behavior a list of positions throughout the path taken at temporal intervals. These functions provide insight into Lehman and Stanley's intuitions about behavior: to them, behavior relates to the state of the process, rather than to the agent or its actions. Because the state of the process depends upon the interaction of the process and the agent, these definitions assure that behavior reflects the interaction between the process and the agent.
		
		This is important. Notice that under the function-approximation framework (\autoref{outputFunction}), any function with appropriate range and domain could be treated as an agent. As a result, if behavior were taken to be a matter of the function approximator alone, one would be forced to accept the premise that the structure of the set of agents should be identical for any pair of processes with the same sets of states and actions. Identical even when the state-transition functions differ. In other words, \emph{all} three-dimensional navigation tasks would have the same space of agents. This issue certainly suffices to explain Lehman and Stanley's attitude toward behavior functions. It does not, however, \emph{necessitate} that approach.
		
		Other authors have considered near-totally general notions of behavior. One group of these focuses on collating as much information as possible about \emph{every} point in a path. In the case of Gomez et al. \cite{Gomez2009}, this involves concatenating some number of observed states. Conti et al. \cite{Conti2017} take a similar approach, replacing the states of the decision process with RAM states - the version of the state of the decision process stored by the computer.
		
		Another class of general notions of behavior focuses instead on the actions of the agent itself, as viewed across a subset of $S$. We call this class of functions \emph{Primitive Behavior}.		
			
		\subsection{Primitive Behavior}\label{primitiveBehavior}
			Primitive behavior functions define the behavior of a strictly Markov agent $\agentdefaut$ as the restriction of the agent to a subset of $S$:
			\begin{definition}[Primitive Behavior]
				A behavior function $\behavior : \agentSet \rightarrow M$ is said to be \emph{primitive} if it is a collection of an agent's actions in response to a finite set of states $X \subset S$; $\behavior$ is primitive iff 
				\begin{align}
					\behavior(\agentdefaut) = \{(s, \agentdefaut(s))\,|\,s \in X\}, \text{ and}
				\end{align}
				the distance on the set of behaviors is given by a weighted (by $w_{s}$) sum of distances between the actions of the agents on $X$:\footnote{This assumes that the set of actions $\actionSet$ is a metric space with $d_{\actionSet}$.}
				\begin{align}\label{primitiveBehaviorEquation}
					d(\behavior(\agentun), \behavior(\agentdeux))=\sum_{s \in X}w_{s}d_{\actionSet}(\agentun(s), \agentdeux(s)).
				\end{align}	
			\end{definition}
		
			This notion of behavior, with slight modifications, has appeared in several papers in the Reinforcement Learning literature \cite{Meyerson2016, Parker-Holder2020, Stork2020, Pacchiano2020}. At least one existing work uses this notion of behavior in Novelty Search \cite{Meyerson2016}. Another \cite{Parker-Holder2020} uses it for optimization with an algorithm other than Novelty Search. \cite{Meyerson2016, Stork2020, Pacchiano2020} weight the constituent distances (i.e. $w_{s}$ is not constant), and \cite{Stork2020} uses primitive behavior to study the relationship between behavior and reward.
			
			A number of important properties of primitive behavior have been described. For example, \cite{Parker-Holder2020} notes that agents with the same behavior may have different parameters. \cite{Meyerson2016} takes implicit advantage of the fact that agents which do not encounter a state do not meaningfully have a response to it (a fact we address in \autoref{identifyingDifferences} of \autoref{pseudometrics}), and \cite{Stork2020} considers the states which an agent encounters an important aspect of the agent, using them to create equivalence classes of agents.
			
			We call this notion of behavior \emph{primitive} because it is the simplest notion of behavior which completely describes the interaction between the agent and the process [on $X$]. Thus, for an appropriate set $X$, the primitive behavior contains all information relevant to $(\mathcal{P}, \agentdefaut)$. So long as a definition of behavior only depends on the interaction between $\mathcal{P}$ and $\agentdefaut$, primitive behavior thus suffices to determine \emph{every} other notion of behavior. 
			
			Clearly, this does not follow the advice of Lehman and Stanley \cite{Lehman2011}; it is completely unfitted to the underlying problem. In exchange for this lack of fit, primitive behavior is \emph{fully} general. Further, because primitive behavior is simply a restriction of the agent itself, \emph{every} other notion of behavior is downstream of primitive behavior, provided that an appropriate set $X$ is used.
			
			However, because the selection of $X$, and of the weights $w_{s}$ is itself a matter of choice, primitive behavior in general remains somewhat subjective. In finite problems, it is possible to assign to every state a non-zero weight, which produces a sort of objective distance, but this remains problematic; two agents might differ on a state which neither of them ever visits. Should agents $\agentun$ and $\agentdeux$ which produce identical processes $(\mathcal{P}, \agentun)$ and $(\mathcal{P}, \agentdeux)$ really be described as different? We contend in \autoref{defAgentIdentity} that the answer is ``no''. 
			
			Our task in the next several sections is to resolve this and other issues with primitive behavior. In the next section we approach the matter of a general substrate for exploration from the ground up. We begin with the simplest version of primitive behavior (that associated with a single state) and proceed to a ``complete'' notion of behavior: a distance between agents which properly discriminates between agents (see \autoref{defAgentIdentity}). In \autoref{theAgentSpace}, we demonstrate some properties of this completed space, which we call the \emph{agent space}.
			
\section{Seeking a Structure for Naïve Exploration}\label{pseudometrics}
	Under \autoref{defNaiveExploration}, naïve exploration is a category of exclusion. Any naïve method which is not exploitative, i.e. which does not use the structure induced by an agent's expected reward, is a method of naïve exploration. Our task in this section is to develop a \emph{good} structure for exploration. That is, to develop a structure on the space of agents, other than reward, which captures important aspects of the relationships of agents to one another and to the process. We seek a structure which:
	\begin{enumerate}[i.]
		\item Exists in every discrete-time decision process,
		\item Correctly describes equivalence relations between agents (see Definitions \ref{metricAxioms} and \ref{defAgentIdentity}), i.e.
			\begin{enumerate}
				\item Identifies agents which differ in aspects irrelevant to their processes,
				\item Distinguishes agents which differ in aspects which \emph{are} relevant to their processes, and \label{identifyingDifferences}
			\end{enumerate}
		\item Naturally describes important relations on the set of agents.
	\end{enumerate}
	Such a structure would allow us to compare structures which are used in naïve exploration methods, including, for example, the various behavior functions which have been used in Novelty Search. If computationally tractable, such a structure could also provide the basis for a new exploration method, or perhaps even a new \emph{definition} of exploration. Let us begin by considering one possible kind of structure for this.

	\subsection{Prototyping the Agent Space}\label{anAgentSpace}
		In contending with a generic discrete-time decision process, few assets are available to define the structure of an agent space. At a basic level, there are only two types of interaction between an agent and a process: the generation of a state, in which the decision process acts on the agent (${\mathcal{P} \rightarrow \agentdefaut}$), and the action of an agent, in which the agent acts on the decision process (${\agentdefaut \rightarrow \mathcal{P}}$). Every other aspect of a decision process may be regarded as a function of those interactions. Let us begin by using these basic interactions to define a metric space, following the behavior functions used in Novelty Search.
		\begin{definition}[Metric]\label{metricAxioms}
			A \emph{metric} on a set $X$ is a function
			\begin{align}
			d:X \times X \rightarrow \mathbb{R}^{+} \cup \{0\}
			\end{align}
			which satisfies the metric axioms:
			\begin{enumerate}
				\item $d(x, y) = 0 \iff x = y$ \hfill Identity of Indiscernibles 
				\item $d(x,y) = d(y,x)$ \hfill Symmetry
				\item $d(x,y) \leq d(x,z) + d(z,y)$ \hfill Triangle Inequality
			\end{enumerate}
		\end{definition}\noindent
		Let us begin with a simple case, comparing strictly Markov agents on their most basic elements: their actions on the process in response to a single state $s$.

	\subsection{The Distance on $s$}
		Let $\actionSet$ be a metric space with metric $d_{\actionSet}$. Then, we define the distance between agents $\agentun$ and $\agentdeux$ in a single-state decision on the state $s$:
		\begin{definition}[Distance on $s$]
			The \emph{distance on $s$} between $\agentun$ and $\agentdeux$ is the distance of their actions on $s$:
			\begin{align}
				d_{s}(\agentun, \agentdeux) = d_{\actionSet}(\agentun(s), \agentdeux(s)).
			\end{align}
		Importantly, this distance is \emph{not} a metric.\footref{notAMetric} Instead, it is a \emph{pseudo}metric; it cannot distinguish agents which act identically on $s$ but differently on another state $s'$.
		\end{definition}
		\begin{definition}[Pseudometric]
			A \emph{pseudometric} on a set $X$ is a function
			\begin{align}
			d:X \times X \rightarrow \mathbb{R}^{+} \cup \{0\}
			\end{align}
			which satisfies the pseudometric axioms
			\begin{enumerate}
				\item $d(x,y) = 0 \impliedby x = y $ \hfill \textbf{Indiscernibility of Identicals}
				\item $d(x,y) = d(y,x)$ \hfill Symmetry
				\item $d(x,y) \leq d(x,z) + d(z,y)$ \hfill Triangle Inequality
			\end{enumerate}
		\end{definition}\noindent
		This distance describes the differences between $\agentun$ and $\agentdeux$ on the state $s$, but decision processes involve many states, potentially infinitely many. Certainly, the action which agents take in response to a single state does not suffice to explain the differences between agents. Let us begin to resolve this by comparing agents on a finite set of states $X$ ($|X| > 1$).
	
	\subsection{The Distance on $X$}
		Having defined the distance between agents on a single state $s$, we can define the distance on a set of states $X$ by summation. Denote the distance between agents on a finite set of states $X$ as $d_{X}(\agentun,\agentdeux)$.
		\begin{definition}[Distance on $X$]
		The \emph{distance on $X \subset S$ between $\agentun$ and $\agentdeux$} is the sum of the distances between $\agentun$ and $\agentdeux$ on each element of the set:
		\begin{align}
			d_{X}(\agentun,\agentdeux) &= \sum_{s \in X}d_{s}(\agentun,\agentdeux)\\
			&= \sum_{s\in X}d_{\actionSet}(\agentun(s), \agentdeux(s)).
		\end{align}
		\end{definition}
	
		Depending upon the process and the set $X$ itself, this could contain \emph{all} of the states in a process (for decision processes with finite sets of states), or a set of \emph{important} states. We can also extend this notion of distance by weighting the distance at each state $s \in X$ with a weight $w_{s}$, producing the primitive behavior of \autoref{primitiveBehavior},
		\begin{align}\tag{\ref{primitiveBehaviorEquation}}
			d(\behavior(\agentun), \behavior(\agentdeux))=\sum_{s \in X}w_{s}d_{\actionSet}(\agentun(s), \agentdeux(s)).
		\end{align}
		Consider a special case for $X$: let $X$ be the set of states observed before time $T$ in a path $\phi$:
		\begin{definition}[The distance on $\phi_{t}$]
			\begin{align}
				d_{\phi_{t}}(\agentun,\agentdeux) = \sum_{t\in\mathbb{Z}_{t}}d_{\phi_{i}(s)}(\agentun,\agentdeux),
			\end{align}
			Which is the distance between $\agentun$ and $\agentdeux$ over a truncated path. When $\phi$ is drawn from $\Phi^{\agentun}$, $d_{\phi_{t}}$ is especially interesting; it is a description of the way that $\agentdeux$ would have acted differently from $\agentun$ over a path that $\agentun$ actually encountered, a sort of backseat-driver metric.
		\end{definition}
		The distance on $\phi_{t}$ is powerful in a number of respects. Let us consider the case $d_{\phi_{t}}(\agentun,\agentdeux) = 0$. Clearly, this implies that $\agentun$ and $\agentdeux$ do not differ at all on this path - presented with the same initial state, they would produce exactly the same truncated path. Unfortunately, the distance at $\phi_{t}$ still fails to satisfactorily distinguish agents - it says nothing about paths in which other states are encountered, or longer paths, or about the stochastic nature of decision processes. Let us state these failings directly so that we may address them:
		
		\begin{remark}[Three Properties]\label{threeProperties}\mbox{} 
			\begin{enumerate}[I.]
				\item The distance on a truncated path $\phi_{t}$ drawn from $\Phi^{\agentun}$ is not reciprocal; it describes how $\agentdeux$ differs from $\agentun$, but not how $\agentun$ differs from $\agentdeux$,\label{item1}
				\item This distance ignores the stochasticity of the process; the ways in which $\agentun$ and $\agentdeux$ differ on $\phi_{t}$ do not necessarily imply anything about the other paths which $\agentun$ experiences,\label{item2}
				\item The distance on a truncated path cannot account for infinite paths; if the agents are not assumed to be strictly Markov, one can easily construct a pair of agents $\agentun$ and $\agentdeux$ which differ only on longer paths. Even in the strictly Markov case, some states might only be possible after $t$.\label{item3}
			\end{enumerate}
		\end{remark}
	
	\subsection{The Role of Time}\label{roleTime}
		It is common in Reinforcement Learning to treat problems with infinite time horizons by weighting sums with a discount function $\omega:\mathbb{N}\rightarrow[0,1]$. If the space of actions $\actionSet$ is bounded,\footnote{A metric space $\actionSet$ is bounded iff $\sup\{d_{\actionSet}(x, y)\,|\, x, y \in \actionSet\} < \infty$} and
		\begin{align}
			\sum_{t\in\mathbb{N}}\omega(t)< \infty \label{finiteDiscountFunction}
		\end{align}
		Then we may define the distance between strictly Markov agents $a$ and $b$ on a complete path:
		\begin{align}
			d_{\phi}(\agentun,\agentdeux) = \sum_{t\in\mathbb{N}}\omega(t)d_{\phi_{t}(s)}(\agentun,\agentdeux).\label{distanceOnPhi}
		\end{align} 
		Clearly, if $d_{\actionSet}$ is bounded, then \cite[60]{Rudin1976}
		\begin{align}
			d_{\phi}(\agentun,\agentdeux) \leq \sup_{\action_{1}, \action_{2} \in \actionSet}\{d_{\actionSet}(\action_{1}, \action_{2})\} \cdot \sum_{t\in\mathbb{N}}\omega(t) < \infty.
		\end{align}
		Thus, this pseudometric resolves the problem of \autoref{item3} of \autoref{threeProperties}, describing the differences in the action of agents over an infinite path.
		
		Our definition of $d_{\phi}(\agentun,\agentdeux)$ readily admits a change that allows the distance on a path to be defined for agents which are not strictly Markov. All we must do is remove the symbols ``$(s)$'':
		\begin{align}
			d_{\phi}(\agentun,\agentdeux) = \sum_{t\in\mathbb{N}}\omega(t)d_{\phi_{t}}(\agentun,\agentdeux).
		\end{align}
		Let us use the distance on $\phi$ to define a notion of distance which incorporates the stochastic aspects of the interaction between an agent and a process, resolving the problem of \autoref{item2} of \autoref{threeProperties}.
		
	\subsection{The Distance at $a$}
		Consider $\Phi^{\agentun}$, the distribution of paths generated by the interaction of an agent $\agentun$ with $\mathcal{P}$. Given a discount function $\omega$ satisfying \eqref{finiteDiscountFunction}, \cite[318]{Rudin1976}
		\begin{align}	
			\expectation_{\phi \sim \Phi^{\agentun}}\left[d_{\phi}(\agentdeux,\agenttrois)\right] = &\expectation_{\phi \sim \Phi^{\agentun}}\left[\sum_{t\in\mathbb{N}}\omega(t)d_{\phi_{t}}(\agentdeux,\agenttrois)\right]\\ 
			=&\expectation_{\phi \sim \Phi^{\agentun}} \left[\lim_{n \rightarrow \infty} \sum_{t < n}\omega(t)d_{\phi_{t}}(\agentdeux,\agenttrois)\right]\\
			= \lim_{n \rightarrow \infty} &\expectation_{\phi \sim \Phi^{\agentun}}\left[\sum_{t < n}\omega(t)d_{\phi_{t}}(\agentdeux,\agenttrois)\right]\\
			= \lim_{n \rightarrow \infty}  \sum_{t < n} &\expectation_{\phi \sim \Phi^{\agentun}}\,\left[\omega(t)d_{\phi_{t}}(\agentdeux,\agenttrois)\right]
		\end{align}
		exists and is bounded above. We call this quantity the \emph{distance at} $\agentun$.
		 
		Because the expectation integrates the distance on $\phi$ over all of the paths of $\agentun$, $d_{\agentun}$ compares $\agentdeux$ and $\agenttrois$ on every part of the process which $\agentun$ experiences. This guarantees that the stochasticity involved in the interaction of $\agentun$ and $\mathcal{P}$ is considered in the comparison. However,  the stochasticity involved in the processes $(\mathcal{P}, \agentdeux)$ and $(\mathcal{P}, \agenttrois)$ may be more relevant to their comparison than that produced by $\agentun$. In the next section, we begin to address this by considering the case $\agentun = \agentdeux$.

	\subsection{$d_{\agentun}(\agentun,\agentdeux)$}\label{daab0}
		In order to state our next result, we must introduce \emph{agent identity}. Agent identity collapses two artificial distinctions between representations of agents caused by the use of function approximators. First, it treats approximators which are differently parameterized but identical as functions (e.g. because of a permutation of the order of parameters, as noted by \cite{Parker-Holder2020}) as identical. Second, in keeping with the methods of \cite{Stork2020} it treats functions which differ, \emph{but only on a set of probability 0} as identical.
		\begin{definition}[Agent Identity ($\agentun \equiv \agentdeux$)]\label{defAgentIdentity}
			Let us say that the agents $\agentun$ and $\agentdeux$ are \emph{identical as agents} ($\agentun \equiv \agentdeux$) if and only if the set of paths where they differ has probability 0 in $\Phi^{\agentun}$;
			\begin{align}
			\mathbb{P}_{\Phi^{\agentun}}\left[\left\{\phi \,| \Exists t(\agentun(\phi_{t}) \neq \agentdeux(\phi_{t}))\right\}\right] &= 0.
			\end{align}
			That is, $\agentun$ and $\agentdeux$ are identical as agents if and only if the probability of $\agentun$ encountering a path where they differ is 0.
		\end{definition}
		\begin{remark}[Agent Identity and $d_{\agentun}(\agentun, \agentdeux)$]
			Notice that this implies $d_{\agentun}(\agentun,\agentdeux) = 0$.
		\end{remark}

		\begin{theorem}[Identical agents the same local distance]\label{identicalDistances}
			If $\agentun$ and $\agentdeux$ are identical as agents, then 
			\begin{align}
				d_{\agentun} = d_{\agentdeux}.
			\end{align}
		\end{theorem}
	
		\begin{proof}[Proof by induction]
			[Base case:] Suppose $d_{\agentun}(\agentun,\agentdeux) = 0$. Then,
			\begin{alignat}{4}
				d_{\agentun}(\agentun, \agentdeux)\,& & = \lim_{n \rightarrow \infty}  \sum_{t < n} &\expectation_{\phi \sim \Phi^{\agentun}}\left[\omega(t)d_{\phi_{t}}(\agentdeux,\agenttrois)\right] = 0\\
				\implies& & \Forall t \in \mathbb{N}\, \biggl(&\expectation_{\phi \sim \Phi^{\agentun}}\left[d_{\phi_{t}}(\agentun,\agentdeux)\right] = 0\biggr)\label{assumeForAllt}\\
				\implies& & &\expectation_{\phi \sim \Phi^{\agentun}}\left[d_{\phi_{0}}(\agentun,\agentdeux)\right] = d_{\sigma_{0}}(\agentun,\agentdeux) = 0.\label{noDifferenceOnPhi0}
			\end{alignat}
			That is, if $d_{\agentun}(\agentun,\agentdeux) = 0$ ($\agentun$ and $\agentdeux$ are identical), then $d_{\sigma_{0}}(\agentun,\agentdeux) = 0$ (they act identically on truncated prime paths of length 0). Then, the joint distributions $(\Phi_{0}^{\prime \agentun}, \agentun(\phi_{0}')) = (\sigma_{0}, \agentun(\phi_{0}'))$ and $(\Phi_{0}^{\prime \agentdeux}, \agentdeux(\phi_{0}')) = (\sigma_{0}, \agentdeux(\phi_{0}'))$ are identical.
			
			Thus, the joint distributions of these and the next state, given by $\sigma$, are also identical: the \emph{total variation distance} of $((\Phi_{0}^{\prime \agentun}, \agentun(\phi_{0}')), \sigma(\Phi_{0}^{\prime \agentun}, \agentun(\phi_{0}'))) = \Phi^{\prime \agentun}_{1}$ and $((\Phi_{0}^{\prime \agentdeux}, \agentdeux(\phi_{0}')), \sigma(\Phi_{0}^{\prime \agentdeux}, \agentdeux(\phi_{0}'))) = \Phi^{\prime \agentdeux}_{1}$ is 0;
			\begin{alignat}{4}
				& &&\,d_{\Phi^{\prime \agentun}_{0}}(\agentun,\agentdeux) &= 0\\
				&\hspace{1ex}\text{and} &&\TVD(\Phi^{\prime \agentun}_{0}, \Phi^{\prime \agentdeux}_{0})\\
				&\hspace{1.4ex} = \hspace{1.8ex} &&\TVD(\sigma_{0}, \sigma_{0}) &= 0\\
				&\implies &&\TVD(\Phi^{\agentun}_{0},\Phi^{\agentdeux}_{0}) &= 0\\
				&\implies &&\TVD(\Phi^{\prime \agentun}_{1}, \Phi^{\prime \agentdeux}_{1})\, &= 0.
			\end{alignat}
			The distribution $\Phi^{\agentun\prime}_{1}$ determines the component of $d_{\agentun}$ at $t = 1$. By \eqref{assumeForAllt}, we have $d_{\Phi_{1}^{\prime \agentun}}(\agentun,\agentdeux) = 0$.
			
			Let $t \in \mathbb{N}$ and suppose that $\TVD(\Phi^{\agentun\prime}_{t}, \Phi^{\agentdeux\prime}_{t}) = 0$. By \eqref{assumeForAllt}, we have $d_{\Phi_{t}^{\prime \agentun}}(\agentun,\agentdeux) = 0$. Then, since $\Phi^{\agentun}_{t}$ and $\Phi^{\agentdeux}_{t}$ are the joint distributions of these, they also have total variation 0, and since $\sigma$ is fixed, the resulting distributions $\Phi^{\prime \agentun}_{t+1}$ and  $\Phi^{\prime \agentdeux}_{t+1}$ also have total variation 0;
			\begin{alignat}{3}
			&d_{\Phi^{\prime \agentun}_{t}}(\agentun,\agentdeux) &&= 0\\
			\text{and}\hspace{0.8ex} &\TVD(\Phi^{\prime \agentun}_{t}, \Phi^{\prime \agentdeux}_{t}) &&= 0\\
			\implies &\TVD(\Phi^{\agentun}_{t},\Phi^{\agentdeux}_{t}) &&= 0\\
			\implies &\TVD(\Phi^{\prime \agentun}_{t+1}, \Phi^{\prime \agentdeux}_{t+1}) &&= 0.
			\end{alignat}
			By assumption \eqref{assumeForAllt}, we have 
			\begin{align}
				d_{\Phi_{t+1}^{\prime \agentun}}(\agentun,\agentdeux) = 0,
			\end{align}
			and thus, the joint distributions also have total variation 0:
			\begin{align}
				\TVD(\Phi^{\agentun}_{t+1}, \Phi^{\agentdeux}_{t+1}).
			\end{align}
			The same holds for all $t \in \mathbb{N}$, so we have
			\begin{align}
				d_{\agentun}(\agentun,\agentdeux) = 0 \iff& \TVD(\Phi^{\agentun}, \Phi^{\agentdeux}) = 0\\
				\implies& \Forall x,y \in \agentSet(d_{\Phi^{\agentun}}(x,y) = d_{\Phi^{\agentdeux}}(x,y))\\
				\iff & d_{\agentun} = d_{\agentdeux}.
			\end{align}
		\end{proof}
		\begin{corollary}[Identical agents are indiscernible under their shared local distance.]\label{indiscernibilityOfIdenticals}
			While the identity of indiscernibles (see \autoref{metricAxioms}) does not generally hold for $d_{\agentun}$, it does hold if
			\begin{enumerate}
				\item We consider the agents \emph{as agents}, and
				\item The distance is taken at one of the considered agents, i.e. $d_{\agentun}(\agentun, \cdot)$;
			\end{enumerate}\vspace{-1ex}
			\begin{align}
				d_{\agentun}(\agentun,\agentdeux) = 0 \iff \agentun \equiv \agentdeux.	
			\end{align}
		\end{corollary}
		\begin{proof}[Proof.]
			By \autoref{identicalDistances}, 
			\begin{align}
				d_{\agentun}(\agentun,\agentdeux) = 0 \iff& \Forall x,y \in \agentSet \left(d_{\agentun}(x,y) = d_{\agentdeux}(x,y)\right)\\
				\implies& d_{\agentdeux}(\agentun,\agentdeux) = d_{\agentun}(\agentun,\agentdeux)\\
				\implies& d_{\agentdeux}(\agentun,\agentdeux) = 0
			\end{align}
		\end{proof}
		\begin{remark}[Symmetry of $d_{\agentun}$]
			Notice that because $d_{\agentun}$ is an integral of distances between agents, which are symmetric, $d_{a}$ is symmetric;
			\begin{align}
				d_{\agentun}(\agentdeux, \agenttrois) = d_{a}(\agenttrois, \agentdeux).
			\end{align}			
		\end{remark}
		\autoref{indiscernibilityOfIdenticals} demonstrates that the agent identity relation of \autoref{defAgentIdentity} is reflexive; for every pair of agents $\agentun, \agentdeux$, $\agentun \equiv \agentdeux \iff \agentdeux \equiv \agentun$.
		
	\subsection{$d_{\agenttrois}(\agentun,\agentdeux) = 0$: When Distance 0 Does Not Imply Agent Identity}\label{dcab0}
		The picture provided above when the distance is taken \emph{at} one of the agents being compared is complicated when the comparison is made from a different vantage point. Let us consider some of these cases:
		\begin{enumerate}
			\item Sometimes, identical agents may be distinguished by a local distance,
			$$\agentun \equiv \agentdeux \land d_{\agenttrois}(\agentun,\agentdeux) > 0.$$
			\item Sometimes, different agents will not be distinguished,
			$$\agentun \not\equiv \agentdeux \land d_{\agenttrois}(\agentun,\agentdeux) = 0.$$
			\item Sometimes, identical distances imply identical agents,\\
			$$d_{\agentun} = d_{\agentdeux} \implies \agentun \equiv \agentdeux.$$ 
			\item Sometimes, identical distances don't,\\
			$$d_{\agentun} = d_{\agentdeux} \land \agentun \not\equiv \agentdeux.$$
		\end{enumerate}
		
		All of these problems stem from the issues of state visitation mentioned in \autoref{primitiveBehavior} and \cite{Stork2020}: unless $\agentun \equiv \agentdeux$, $\agentun$ may visit paths which $\agentdeux$ does not, and vice versa.
		
		\begin{example}[$\agentun \equiv \agentdeux$, but $d_{\agenttrois}(\agentun,\agentdeux) > 0$]
			In general, the functions $\agentun$ and $\agentdeux$ can be identical as agents, while they differ in their responses to unvisited paths. Then, from the perspective of an agent which visits such paths, $\agentun$ and $\agentdeux$ appear different.
		\end{example}
		\begin{example}[$\agentun \not\equiv \agentdeux$, but $d_{\agenttrois}(\agentun,\agentdeux) = 0$]
			Similarly, it is possible for $\agentun$ and $\agentdeux$ to be identical on every path which $\agenttrois$ visits, but differ when $\agentun$ or $\agentdeux$ control the process.
		\end{example}
		\begin{example}[$d_{\agentun} = d_{\agentdeux} \implies \agentun \equiv \agentdeux$]
			In general, local distances form a bijection with the distributions $\Phi^{\agentun}$ and the processes $(\mathcal{P}, \agentun)$, and thus $d_{\agentun} = d_{\agentdeux} \iff \Phi^{\agentun} = \Phi^{\agentdeux} \iff (\mathcal{P}, \agentun) = (\mathcal{P}, \agentdeux) \iff \agentun \equiv \agentdeux$. 
		\end{example}
		\begin{example}[$d_{\agentun} = d_{\agentdeux}$, but $\agentun \not\equiv \agentdeux$]
			However, when the set of agents is restricted, this is not necessarily true. If, for example, agents are assumed to be strictly Markov (see \autoref{strictlyMarkov}), then all that matters to the equivalence of the functions $d_{\agentun}$ and $d_{\agentdeux}$ is the states which they visit; if $\sigma(\phi_{t}', \action_{1})$ = $\sigma(\phi_{t}', \action_{2})$, $\action_{1} \neq \action_{2}$, then agents $\agentun$ and $\agentdeux$ which differ in their response to $\phi_{t}'$ could nonetheless produce ``identical'' distances, when the range of the distance is restricted to pairs of strictly Markov agents.
		\end{example}
		With so few assurances, the local distances may seem pointless. Are they nothing more than markers of identity? No, they are much more; \autoref{convergence} demonstrates that $d_{\agentun}(\agentun,\agentdeux) = 0 \iff \agentun \equiv \agentdeux$ is not a special case. Instead, the local distances \emph{themselves} are continuous in the agent space: as $\agentun$ and $\agentdeux$ approach one another under either local distance, i.e. as $|d_{\agentun}(\agentun,\agentdeux)|$ or $|d_{\agentdeux}(\agentun,\agentdeux)|$ goes to 0, and \emph{so does} $\sup_{x,y \in \agentSet}|d_{\agentun}(x,y) - d_{\agentdeux}(x,y)|$.

\section{The Agent Space}\label{theAgentSpace}
	The collection of local distances described in \autoref{pseudometrics} is an odd basis for the structure of an agent space; rather than a single, objective notion of distance, each agent $\agentdefaut$ defines its own local distance $d_{\agentdefaut}$. When paired with the set of agents, each local distance defines a pseudometric space $(\agentSet, d_{\agentdefaut})$, which describes the ways that agents differ \emph{on} $\Phi^{\agentdefaut}$.
	
	\autoref{daab0} establishes relationships between local distances, but only in the case of identical agents which differ as functions. We have not yet related the local distances of \emph{non}-identical agents. In particular, we have not established that a collection of local distances defines a single space.
	
	One interpretation of the collection of local distances is as a premetric (see \autoref{premetrics}), in a manner analogous to the Kullback-Leibler Divergence. However, $d_{\agentdefaut}$ can also be treated as more than a premetric; it need not be asymmetrical, nor need it violate the triangle inequality, because each agent defines a local distance that describes an internally-consistent pseudometric space. We continue this discussion in greater detail in \autoref{premetrics}, employing the premetric to provide a simple topology equivalent to that defined by convergence in the agent space (\autoref{defAgentConvergence}).
	
	In the next section, we unify the pseudometric spaces produced by each local distance to create an objective \emph{agent space}, whose topology is compatible with many important aspects of Reinforcement Learning, including standard function approximators (e.g. neural networks) and standard formulations of reward (see \autoref{uniformConvergenceofFunctions}).
	
	\subsection{Convergence in Agent Spaces}\label{convergence}
		\autoref{identicalDistances} proves that identical agents have the same local distance,
		\begin{align}
			\agentun \equiv \agentdeux \implies &d_{\agentun} = d_{\agentdeux}.
		\end{align}
		\autoref{indiscernibilityOfIdenticals} gives an important condition for equivalence: agents are identical, and thus have identical local distances, whenever $d_{\agentun}(\agentun,\agentdeux) = 0$. The next step in our analysis of the local distance is to consider the case where $\agentun$ and $\agentdeux$ are close to one another, but their distance is greater than 0. Consider the case
		\begin{align}
			0 < d_{\agentun}(\agentun,\agentdeux) < \delta,
		\end{align}
		with $\delta > 0$. In order to simplify the remainder of this section, we restrict ourselves to \emph{stochastic} agents. Let the metric on $\actionSet$, $d_{\actionSet}$, be the total variation distance $\TVD(\action_{1}, \action_{2})$. In this case, the logic of \autoref{identicalDistances} can be extended. \autoref{identicalDistances} demonstrates that two agents which are at every time \emph{identical} must produce identical distributions of paths, and, as a result, identical local distances. Consider a pair of agents $\agentun$ and $\agentdeux$, which have a distance less than $\delta$ on $\Phi^{\agentun}$, $d_{\agentun}(\agentun,\agentdeux) < \delta$, with $\omega(t) = 1$. Then, 
		\begin{alignat}{4}
			&&0 < d_{\Phi^{\prime \agentun}_{0}}(\agentun,\agentdeux) \leq \hspace{0.42em} d_{\agentun}(\agentun,\agentdeux) &< \delta\\
			\implies&& d_{\Phi^{\prime \agentun}_{0}}(\agentun,\agentdeux) = d_{\sigma_{0}}(\agentun,\agentdeux) &< \delta\\
			\iff&& \expectation_{\phi_{0}' \sim \sigma_{0}}\left[\TVD(\agentun(\phi_{0}'), \agentdeux(\phi_{0}'))\right] &< \delta\\
		\end{alignat}
		Since $\sigma_{0}$ does not vary with the agent, we have
		\begin{align}
			\Phi^{\prime \agentun}_{0} = \Phi^{\prime \agentdeux}_{0} \land \expectation_{\phi_{0}' \sim \sigma_{0}}\left[\TVD(\agentun(\phi_{0}'), \agentdeux(\phi_{0}'))\right] < \delta\\
			\implies \TVD(\Phi_{1}^{\prime \agentun}, \Phi_{1}^{\prime \agentdeux}) < \delta.\label{boundOnVarianceAt1}
		\end{align}
		Likewise,
		\begin{align}
			d_{\Phi_{1}^{\prime \agentun}}(\agentun, \agentdeux) \leq d_{\agentun}(\agentun, \agentdeux) < \delta
		\end{align}
		should imply that
		\begin{align}
			\TVD(\Phi_{2}^{\prime \agentun}, \Phi_{2}^{\prime \agentdeux}) < \delta,
		\end{align}
		except that since 
		\begin{align}
			\TVD(\Phi_{1}^{\prime \agentun}, \Phi_{1}^{\prime \agentdeux}) < \delta,\tag{\ref{boundOnVarianceAt1}}
		\end{align}
		we must start from a baseline of $\delta$, giving the bound $2\delta$. In general, we have
		\begin{align}
			\TVD(\Phi_{t}^{\prime \agentun}, \Phi_{t}^{\prime \agentdeux}) < t\delta.
		\end{align}
		This bound can be improved by noting that the total variation $\TVD(\Phi_{1}^{\prime \agentun}, \Phi_{1}^{\prime \agentdeux})$ can be bounded above by (and is in fact equal to) the smaller quantity
		\begin{align}
			d_{\Phi_{0}^{\prime \agentun}}(\agentun,\agentdeux) = d_{\sigma_{0}}(\agentun,\agentdeux),
		\end{align}
		yielding in the general case
		\begin{align}
			\TVD(\Phi_{t}^{\prime \agentun}, \Phi_{t}^{\prime \agentdeux}) < \sum_{i \in \mathbb{Z}_{t}}d_{\Phi_{i}^{\prime \agentun}}(\agentun,\agentdeux).
		\end{align}
		With our assumption that $\omega(t) = 1$, we can bound the right side of this inequality above, giving the looser inequality
		\begin{align}
			\TVD(\Phi_{t}^{\prime \agentun}, \Phi_{t}^{\prime \agentdeux}) < \sum_{i \in \mathbb{Z}_{t}}d_{\Phi_{i}^{\prime \agentun}}(\agentun,\agentdeux) < d_{a}(a,b) &< \delta\\
			\implies \TVD(\Phi_{t}^{\prime \agentun}, \Phi_{t}^{\prime \agentdeux}) &< \delta.
		\end{align}
		If the discount function is not the constant value $1$ (i.e. if $\omega(t) < 1$ for some $t$), as assumed above, the sum above gains a factor of $\frac{1}{\omega(i)}$:
		\begin{align}
			\TVD(\Phi_{t}^{\prime \agentun}, \Phi_{t}^{\prime \agentdeux}) < \sum_{i \in \mathbb{Z}_{t}}\frac{1}{\omega(i)}d_{\Phi_{i}^{\prime \agentun}}(\agentun,\agentdeux).
		\end{align}
		For simplicity we now assume that $\omega(t) = \gamma^{t}, 0 < \gamma < 1$, though the following results apply to a more general family of functions (for example, they apply to all monotonic super-exponential decay functions). 
		\begin{lemma}[$\TVD(\Phi_{t}^{\agentun}, \Phi_{t}^{\agentdeux})$ can be bound above by a function of $d_{\agentun}(\agentun,\agentdeux)$]\label{bindTotalVariationByLocalDistanceLemma}
			Notice that when 
			\begin{align}
				s < t \implies \omega(s) \geq \omega(t) \implies \frac{1}{\omega(s)} \leq \frac{1}{\omega(t)},
			\end{align} 
			so for a fixed distance $d_{\agentun}(\agentun,\agentdeux)$, the maximal total variation $\TVD(\Phi_{t}^{\agentun}, \Phi_{t}^{\agentdeux})$ is achieved when
			\begin{align}
				\Forall s \leq t\left(d_{\Phi^{\prime \agentun}_{s}}(\agentun,\agentdeux) = 0\right) \land  d_{\Phi^{\prime \agentun}_{t}}(\agentun,\agentdeux) = d_{\agentun}(\agentun,\agentdeux).
			\end{align}
			Thus, we can bound the total variation $\TVD(\Phi_{t}^{\agentun}, \Phi_{t}^{\agentdeux})$ above by $\frac{1}{\gamma^{t}}d_{\agentun}(\agentun,\agentdeux)$. Thus,
			\begin{align}
				d_{\agentun}(\agentun,\agentdeux) < \delta \implies \TVD(\Phi_{t}^{\agentun}, \Phi_{t}^{\agentdeux}) < \frac{1}{\gamma^{t}}\delta.\label{bindTotalVariationByLocalDistance}
			\end{align}
		\end{lemma}
		\autoref{bindTotalVariationByLocalDistanceLemma} enables us to prove our next theorem, the limit equivalent of \autoref{indiscernibilityOfIdenticals}. Let us begin with a definition.
		\begin{definition}[Convergence in the Agent Space]\label{defAgentConvergence}
			We say that a sequence of agents $x_{n}$ converges to an agent $\agentdefaut$ if and only if the local distance between the agents in the sequence and $\agentdefaut$ goes to 0;
			\begin{align}
				x_{n} \rightarrow \agentdefaut \iff \Forall \varepsilon > 0 \, \Exists m \in \mathbb{N} \,(n > m \implies d_{\agentdefaut}(\agentdefaut,x_{n}) < \varepsilon).
			\end{align}
			
		\end{definition}
		\begin{theorem}[The Limit Behavior of Local Distances]\label{convergingAgentsConvergingDistances}
			Let $x_{n}$ be a sequence of agents converging to $\agentun$. Then,
			\begin{enumerate}
				\item $\Forall t \in \mathbb{N}(\lim_{x_{n}\rightarrow \agentun}\TVD(\Phi_{t}^{x_{n}},\Phi_{t}^{\agentun}) = 0)$,\label{agentConvergenceImpliesMeasureConvergenceForEachT}
				\item $d_{x_{n}} \rightarrow d_{\agentun}$, and \label{agentConvergenceImpliesDistanceConvergence}
				\item $d_{x_{n}}(x_{n}, \agentun) \rightarrow 0$.
			\end{enumerate}
		\end{theorem}
		\begin{proof}[Proof of \ref{agentConvergenceImpliesMeasureConvergenceForEachT}]
			By \eqref{bindTotalVariationByLocalDistance}, we have for any fixed $t$ and any agent $x_{n}$
			\begin{align}
				d_{\agentun}(\agentun,x_{n}) < \delta \implies \TVD(\Phi_{t}^{\agentun}, \Phi_{t}^{x_{n}}) < \frac{1}{\gamma^{t}}\delta.
			\end{align}
			By assumption, for every $\delta > 0$ there is an $m \in \mathbb{N}$ with $n > m \implies d_{\agentun}(\agentun,x_{n}) < \delta$. For any $\varepsilon > 0$, there is a $\delta_{\varepsilon} > 0$ with $\frac{1}{\gamma^{t}}\delta_{\varepsilon} < \varepsilon$. Thus, we can select an $n \in \mathbb{N}$ with
			\begin{alignat}{4}
				m > n & \implies& d_{\agentun}(\agentun,x_{n}) < &\,\delta_{\varepsilon}\\
				& \implies& \TVD(\Phi_{t}^{\agentun}, \Phi_{t}^{x_{n}}) < &\,\varepsilon.
			\end{alignat}
		\end{proof}
		\begin{proof}[Proof of \ref{agentConvergenceImpliesDistanceConvergence}]
			Per the proof of \ref{agentConvergenceImpliesMeasureConvergenceForEachT}, we have for any $\varepsilon > 0$ and any $t \in \mathbb{N}$ an $m$ giving $n > m \implies \TVD(\Phi_{t}^{\agentun}, \Phi_{t}^{x_{m}}) < \varepsilon$. Let $\actionSet_{\text{max}}$ be the bound on $\actionSet$ (in the case of total variation, $\actionSet_{\text{max}} = 1$). For each path $\phi$, we have
			\begin{align}
				d_{\phi}(\agentun,\agentdeux) - d_{\phi_{t}}(\agentun,\agentdeux) \leq& \sum_{i =  t+1}^{\infty}\gamma^{i}\,\actionSet_{\text{max}}\\
				= \frac{\gamma^{t+1}}{1 - \gamma}\actionSet_{\text{max}} =& \frac{\gamma^{t+1}}{1 - \gamma},
			\end{align}
			and we have the analogous bound for $\Phi^{\agentun}$
			\begin{align}
				d_{\agentun} - \sum_{i=0}^{i=t}d_{\Phi_{i}^{\agentun}} \leq \frac{\gamma^{t+1}}{1 - \gamma}.
			\end{align}
			\begin{remark}[Notation for the Distance on Distributions of Truncated Paths]
				We now need to manipulate terms of this type, for which a bit of notation will be useful: Let
				\begin{alignat}{3}
					&d_{\agentun}^{t}\! &&= \sum_{i=0}^{i=t} d_{\Phi_{i}^{\agentun}}, \text{ and}\\
					&d_{\agentun}^{t+}\! &&= \hspace{-0.3em}\sum_{i=t+1}^{i=\infty}\hspace{-0.35em} d_{\Phi_{i}^{\agentun}}.
				\end{alignat}
				Further, for any agent $\agentun$ we can decompose $d_{\agentun}$ into
				\begin{alignat}{4}
					d_{\agentun}& = \sum_{i=0}^{t} &&d_{\Phi_{t}^{\agentun}} &&+ \sum_{i=t+1}^{\infty}d_{\Phi_{t}^{\agentun}}\\
					d_{\agentun}& = &&d_{\agentun}^{t} &&+ (d_{\agentun} - d_{\agentun}^{t})\\
					d_{\agentun}& = &&d_{\agentun}^{t} &&+ d_{\agentun}^{t+}.
				\end{alignat}
			\end{remark}\noindent
			Notice that the maximum value of $d_{\agentun}^{t}$ is
			\begin{align}
				\sum_{i=0}^{t}\gamma^{t}\actionSet_{\text{max}} = \frac{1 - \gamma^{t + 1}}{1 - \gamma}\actionSet_{\text{max}},
			\end{align}
			and we can bound $d_{\agentun}^{t+}$ from above as well,
			\begin{align}
			d_{\agentun}^{t+} \leq& \,\max(d_{\agentun}) - \max(d_{\agentun}^{t})\\
			=& \,\frac{1}{1 - \gamma}\actionSet_{\text{max}} - \frac{1 - \gamma^{t + 1}}{1 - \gamma}\actionSet_{\text{max}}\\
			=& \,\frac{\gamma^{t+1}}{1-\gamma}\actionSet_{\text{max}}.
			\end{align}
			Clearly, as $t\rightarrow \infty$, this bound goes to 0.
			
			Let $\varepsilon > 0$, $\varepsilon' = \frac{\varepsilon(1 - \gamma)}{2\actionSet_{\text{max}}(1 - \gamma^{t + 1})}$, and let $\delta_{\varepsilon'}$ be as above. Then, we have
			\begin{alignat}{3}
				&& d_{\agentun}(\agentun,\agentdeux) &< \delta_{\varepsilon'}\\
				&\implies& \TVD(\Phi_{t}^{\agentun}, \Phi_{t}^{\agentdeux}) &< \varepsilon' = 
				\frac{\varepsilon(1 - \gamma)}{\actionSet_{\text{max}}(1 - \gamma^{t + 1})}\\
				&\implies& d_{\agentun}^{t} - d_{\agentdeux}^{t} &< \frac{1 - \gamma^{t + 1}}{1 - \gamma}\actionSet_{\text{max}}\TVD(\Phi_{t}^{\agentun}, \Phi_{t}^{\agentdeux})\\
				&& &< \frac{1 - \gamma^{t + 1}}{1 - \gamma}\actionSet_{\text{max}}\varepsilon' =  \frac{\varepsilon}{2}.
			\end{alignat}
			Thus, we have
			\begin{align}
				d_{\agentun} - d_{\agentdeux} &= d_{\agentun}^{t} + d_{\agentun}^{t+} - (d_{\agentdeux}^{t} + d_{\agentdeux}^{t+})\label{boundda-db}\\
				&\leq d_{\agentun}^{t} - d_{\agentdeux}^{t} + 2\left[\actionSet_{\text{max}}\frac{\gamma^{t}}{1-\gamma}\right]\\
				&< \frac{\varepsilon}{2} + 2\left[\actionSet_{\text{max}}\frac{\gamma^{t}}{1-\gamma}\right].\label{almost2epsilon}
			\end{align}
			Now, set $t$ great enough that $2\left[\actionSet_{\text{max}}\frac{\gamma^{t}}{1-\gamma}\right] < \frac{\varepsilon}{2}$, and set $\agentdeux = x_{n}$. Per part \ref{agentConvergenceImpliesMeasureConvergenceForEachT}, we can select an $m$ with $n > m \implies d_{\agentun}(\agentun,x_{n}) < \delta_{\frac{\varepsilon}{2}}$. Then, combining lines \eqref{boundda-db} and \eqref{almost2epsilon}, we have
			\begin{align}
				d_{\agentun} - d_{x_{m}} < \varepsilon.
			\end{align}
		\end{proof}
		\begin{proof}[Proof of 3.]
			Applying \ref{agentConvergenceImpliesDistanceConvergence}, we have
			\begin{align}
				d_{x_{n}} \rightarrow d_{\agentun}.
			\end{align}
			Since $d_{\agentun}(\agentun,x_{n}) = 0$, we have
			\begin{align}
				d_{x_{n}}(\agentun,x_{n}) \rightarrow d_{\agentun}(\agentun,x_{n}) = 0.
			\end{align}
		\end{proof}
		
		This theorem demonstrates that agents which are close in the agent space have close perspectives and produce close local distances. In fact, the proof of \autoref{agentConvergenceImpliesDistanceConvergence} of \autoref{convergingAgentsConvergingDistances} demonstrates that the local distances $d_{\agentdefaut}$ which represent those perspectives are \emph{uniformly continuous} in the agent. Further, \autoref{agentConvergenceImpliesMeasureConvergenceForEachT} of \autoref{convergingAgentsConvergingDistances} demonstrates that similar agents produce similar distributions of truncated paths - not just similar distance functions.
		
		In the next section we consider a loose method of interpreting the local distances: the interpretation of the local distance as a function of \emph{two}, rather than three, agents, fixing the vantage point at the first agent being compared. This allows us to describe the local distance as a \emph{premetric}. We use this fact to define the topology of the agent space in \autoref{agentSpaceTopology}.
	\subsection{$d_{x}(x,y)$ as a Premetric}\label{premetrics}
		A premetric is a generalization of a metric which relaxes several properties, giving the very general definition
		\begin{definition}[Premetric]
			A function $D: X \rightarrow \mathbb{R}^{+} \cup \{0\}$ is called a premetric if \cite[23]{Arkhangelskii1990}
			\begin{align}
				D(x,x) = 0.
			\end{align}
			such a premetric is called \emph{separating} if it also satisfies
			\begin{align}
				D(x,y) = 0 \iff x = y.
			\end{align}
		\end{definition}
		Many important functions satisfy this definition, including the Kullback-Liebler Divergence. 
		\begin{remark}[The Local Distance is a Premetric]
			Notice that the function 
			\begin{align}
			D(\agentun,\agentdeux) = d_{\agentun}(\agentun,\agentdeux)
			\end{align}
			is a separating premetric.
		\end{remark}
		Important for the practical use of the local distances, this premetric (along with the other structures of the agent space) is able to describe the differences between agents \emph{and} between the distributions of paths which they produce without actually sampling those distributions; $d_{\agentun}(\agentdeux,\agenttrois)$ compares the distributions $\Phi^{\agentdeux}$ and $\Phi^{\agenttrois}$ but only requires information about the distribution $\Phi^{\agentun}$ the functions $\agentdeux$ and $\agenttrois$. This is valuable because in Reinforcement Learning it is typically simple to calculate the actions which an agent would take from that agent's parameters, but information about the distribution usually needs to be sampled - an expensive operation. This is especially valuable if many nearby agents need to be compared (e.g. because the agents being considered are based on a single locus agent). Operations which involve comparing a pair of agents using the standard of a third like this are common in Reinforcement Learning. For example, the $\mathcal{Q}^{\agentun}$ \autoref{defActionValueFunction} function is often used to judge the quality of the actions of other agents $\agentdeux \neq \agentun$.
		
		In the next section we describe a topology on the agent space which we will take as canonical (i.e. as \emph{the} topology of the agent space). There are two basic ways to understand the topology: it may be understood as the topology of the premetric space given by the premetric on the agent space described above, or it may be understood as the topology given by the \emph{convergence} relation of \autoref{defAgentConvergence}. These are identical. In fact, \autoref{defAgentConvergence} can be defined using only the premetric description of the local distance.
		
	\subsection{The Topology of the Agent Space}\label{agentSpaceTopology}
		Let us start by providing two equivalent definitions of the topology of the agent space: one definition of its open sets, and another definition of its closed sets.
		\begin{definition}[The Topology of $(\agentSet, d)$: Open Sets]
			We say that a set $U \in \agentSet$ is open if and only if about every point $x \in U$, $U$ admits an open disk of positive radius:
			\begin{align}
				\Forall x \in U \Exists r > 0 \,(\{y \,|\, d_{x}(x,y) < r\} \subset U).\label{defOpenSetsOfAgentSpaces}
			\end{align}
		\end{definition}
		
		\begin{definition}[The Topology of $(\agentSet, d)$: Closed Sets]		
			We say that a set $X \in \agentSet$ is closed if and only if it contains its limit points; iff for every convergent sequence $x_{n}$ with $x_{n} \in \agentSet$: $X$ is closed iff
			\begin{align}
				\Forall x_{n}(\Forall n \in \mathbb{N}\,(x_{n} \in X) \implies \lim x_{n} \in X).
			\end{align}
		\end{definition}
		These definitions suffice, in fact, to define the topology of any premetric (or metric). It may be demonstrated that these definitions produce the same topology (for example, by remarking that open sets in metric spaces may be characterized by the criterion \eqref{defOpenSetsOfAgentSpaces}). It is important to note that this, along with the premetric version of the agent space, represents a sort of lower-bound on the structure which the local distances describe on the set of agents. In particular, the local distances may prove useful beyond simple problems of limits. In \cite{Allen2021}, we employ the local distances for exploration in an implementation of Novelty Search.
	
		In the next section, we demonstrate that the topology of the agent space is compatible with many of the most important aspects of Reinforcement Learning. In particular, we show that standard formulations of reward are continuous in the agent space, and that the agent space itself is continuous in the parameters of most agent approximators, demonstrating that the agent space is a valid structure for the set of agents, and for Reinforcement Learning more generally.
		
\section{Functions of the Agent}\label{continuity}
	The topology of the agent space carries information about many important aspects of the decision process and its interaction with agents, including the distributions of truncated paths. However, we have not yet demonstrated any relationship between the agent space and the object of Reinforcement Learning: the expected reward of the agent, $J(\agentdefaut)$. In this section we demonstrate that an important class of reward functions (summable reward functions, \eqref{sumReward}) are continuous functions of the agent in the topology of the agent space. We begin with a simple condition for the agent to be a continuous function of the parameters of a function approximator. We then use the continuity of finite distributions of paths established in \autoref{agentConvergenceImpliesMeasureConvergenceForEachT} of \autoref{convergingAgentsConvergingDistances} in the agent space to prove that the expectation of reward is a continuous function of the agent.
	
	\subsection{Parameterized Agents}\label{parameterizedAgents}
		Let $f$ be a function approximator parameterized by a set of real numbers $\theta$, taking truncated prime paths into a set of actions. Then,
		\begin{align}
			f:\mathbb{R}^{n} \times \Phi^{\prime} \rightarrow \actionSet.
		\end{align}
		If we delay the selection of the truncated path, we may understand $f$ as a function from $\mathbb{R}^{n}$ into the set of agents:
		\begin{align}
			f:\mathbb{R}^{n} \rightarrow \actionSet^{\Phi^{\prime}}
			\iff f:\mathbb{R}^{n} \rightarrow \agentSet.
		\end{align}
		Notice that we have returned to the pre-quotient notion of an agent - the set of functions from the set of truncated prime paths to the set of actions before the equivalence relation of \autoref{defAgentIdentity} is applied. To better distinguish these functions, let us denote the pre-quotient set $F$. The matter of demonstrating that a particular function approximator is a continuous function from its parameters to the agent space may be divided into two parts: it must be demonstrated that the function approximator is a continuous function from the set of parameters to the set of functions, and it must be demonstrated that the quotient operation itself is a continuous function from the set of functions to the set of agents. We begin by demonstrating the continuity of the quotient operation.
		
		Let us assume the $L^{\infty}$ metric on the set of functions and denote the map taking a function $f$ to an agent $\agentdefaut$ by $Q: (F, L^{\infty}) \rightarrow (\agentSet, d)$. Then, the quotient operation which takes the set of agents to the space of agents is continuous if and only if for every convergent sequence $x_{n}$ in $(F, L^{\infty})$, $Q(x_{n})$ converges.
		\begin{theorem}[The Agent Identity Quotient Operation is Continuous]
			The quotient operation defined by \autoref{defAgentIdentity} is a continuous function from the $F$ to $\agentSet$.
		\end{theorem}
		\begin{proof}
			Let $x_{n}$ be a $L^{\infty}$-convergent sequence of functions converging to $x$
			\begin{gather}
				x_{n} : \Phi^{\prime} \rightarrow \actionSet,\\
				\Forall \varepsilon > 0 \Exists m \in \mathbb{N} \Forall n > m\,(d^{\infty}(x_{n}, x) < \varepsilon).
			\end{gather}
			Then, we must show that
			\begin{align}
				\Forall \varepsilon > 0 \Exists m \in \mathbb{N} \Forall n > m\,(d_{Q(x)}(Q(x_{n}), Q(x)) < \varepsilon).
			\end{align}
			Consider the definition of $d_{x}$:
			\begin{align}
				d_{x} =& \sum_{t \in \mathbb{N}}\omega(t) d_{\Phi^{\prime x}_{t}}(x_{n},x)\\
				d_{\Phi^{\prime x}_{t}}(x_{n}, x) =& \expectation_{\phi'_{t} \sim \Phi^{\prime x}_{t}}\left[d_{\phi'}(x_{n},x)\right].
			\end{align}
			Clearly, we have
			\begin{align}
				d_{\Phi^{\prime x}_{t}}(x_{n}, x) &\leq \sup_{\phi' \in \Phi^{\prime}}d_{\phi'}(x_{n},x)\\
				&= d^{\infty}(x_{n}, x).
			\end{align}
			Thus, $d_{x}(x_{n}, x) \leq \sum_{t \in \mathbb{N}}\omega(t) d^{\infty}(x_{n}, x)$. Recall that $\Omega = \sum_{t \in \mathbb{N}}\omega(t) < \infty$. Thus,
			\begin{align}
				\Forall \varepsilon > 0 \Exists m \in \mathbb{N}\,\Forall n > m (d_{x}(x_{n}, x) \leq \varepsilon)
			\end{align}
			exists because an integer $m$ which satisfies
			\begin{align}
				\Forall \frac{\varepsilon}{\Omega} > 0 \Exists m \in \mathbb{N}\, \Forall n > m\left(d^{\infty}(x_{n}, x) < \frac{\varepsilon}{\Omega}\right)
			\end{align}
			exists, by assumption.
		\end{proof}
		To finish the demonstration that a particular function approximator gives agents continuous in its parameters, then, it remains only to show that the function approximator is $L^{\infty}$-continuous (uniformly continuous) in the parameters of the approximator. One class of function approximators which satisfies this is feedforward neural networks, such as those discussed in \cite{Hornik1989}.\footnote{Specifically, neural networks with continuous, bounded activation functions are uniformly continuous in their parameters.}
		
	\subsection{Reward and the Agent Space}
		In order for the agent space to be useful for the \emph{problem} of Reinforcement Learning, it must be related to the object of Reinforcement Learning: \emph{reward}.
		\begin{align}
			R : \Phi \rightarrow \mathbb{R}\tag{\ref{defRewardEquation}}
		\end{align}
		We noted in \autoref{defReward} that reward can frequently be described by a sum,
		\begin{align}
			R(\phi) =& \sum_{t \in \mathbb{N}} r(\phi_{t}(s), \phi_{t}(\action)).\tag{\ref{sumReward}}
		\end{align}
		We also noted that this sum is often weighted by a discount function $\omega(t)$. Discount functions are employed because they offer general conditions under which the reward of a path (and thus its expectation) is bounded: so long as $\Omega$ is finite and the immediate reward $r$ is bounded, so too is the sum \eqref{sumReward}.
		
		This formulation of reward has several valuable properties which can be extracted: the reward function can be extended from $\Phi$ to include truncated paths:
		\begin{align}
			&R: \Phi \cup \bigcup_{t \in \mathbb{N}}\Phi_{t} \rightarrow \mathbb{R}\\
			&R(\phi_{t}) = \sum_{i < t} r(\phi_{t}(s), \phi_{t}(\action)).
		\end{align}
		Clearly, for any path $\phi$ for which $R(\phi)$ exists we have
		\begin{align}
			\lim_{t \rightarrow \infty} R(\phi_{t}) = R(\phi).
		\end{align}
		If the immediate reward $r$ is bounded and the sum is weighted by a discount function $\omega$ with finite sum $\Omega$ then $R$ is bounded and we have the stronger condition
		\begin{align}
			\Forall \varepsilon > 0 \Exists t \in \mathbb{N} \Forall \phi \in \Phi \left(|R(\phi) - R(\phi_{t})| < \varepsilon \right).\label{uniformConvergenceofFunctions}
		\end{align}
		That is, such a summable discount function converges uniformly to its value as $t \rightarrow \infty$.
		
		\begin{theorem}[Functions Continuous in the Agent Space]
			Let $R$ be a bounded real function of the set of paths and truncated paths, and let $J$ be the expectation of $R$ on the distribution of paths $\Phi^{\agentdefaut}$,
			\begin{gather}
				J(\agentdefaut) = \expectation_{\phi \sim \Phi^{\agentdefaut}}\left[R(\phi)\right],\\
				R: \Phi \cup \bigcup_{t \in \mathbb{N}}\Phi_{t} \rightarrow \mathbb{R},\\
				\sup_{\phi, \varphi \in \Phi \cup \bigcup_{t \in \mathbb{N}}\Phi_{t}} |R(\phi) - R(\varphi)| = \overline{R},
			\end{gather}
			and let $R$ satisfy \eqref{uniformConvergenceofFunctions}
			\begin{align}
				\Forall \varepsilon > 0 \Exists t \in \mathbb{N} \Forall \phi \in \Phi \left(t' > t \implies |R(\phi) - R(\phi_{t'})| < \varepsilon \right)\tag{\ref{uniformConvergenceofFunctions}}.
			\end{align}
			Then, $J$ is a continuous function with respect to the agent space $(\agentSet, d)$.
		\end{theorem}
		\begin{proof}			
			Let us demonstrate that $J$ is a continuous function of $\agentdefaut$ by showing that for any convergent sequence $x_{n}$ converging to $\agentdefaut$,
			\begin{align}
				\lim_{n\rightarrow \infty} J(x_{n}) = J(\agentdefaut).
			\end{align}
			Thus, our goal is to demonstrate that
			\begin{align}
				\Forall \varepsilon \Exists m \in \mathbb{N} \left(n > m \implies |J(a) - J(x_{n})| < \varepsilon\right).
			\end{align}
			By assumption of \eqref{uniformConvergenceofFunctions},
			\begin{align}
				\Exists t \in \mathbb{N}\, \Forall \phi \in \Phi\left(t' > t \implies |R(\phi) - R(\phi_{t})| < \frac{\varepsilon}{3}\right).\label{eventuallyEpsilonOverTwo}
			\end{align}
			Consider the expectation of the reward of the truncated paths of $\agentdefaut$,
			\begin{align}
				\expectation_{\phi \sim \Phi^{\agentdefaut}}R(\phi_{t}).
			\end{align}
			By \eqref{eventuallyEpsilonOverTwo}, for an appropriate value of $t$ we have
			\begin{align}
				|R(\phi) - R(\phi_{t})| < \frac{\varepsilon}{3}.
			\end{align}
			Thus, we have
			\begin{align}
				&\left|J(\agentdefaut) - \expectation_{\phi \sim \Phi^{\agentdefaut}}R(\phi_{t})\right| = \left|\expectation_{\phi \sim \Phi^{\agentdefaut}}R(\phi) - \expectation_{\phi \sim \Phi^{\agentdefaut}}R(\phi_{t})\right|\\
				= &\expectation_{\phi \sim \Phi^{\agentdefaut}}|R(\phi) - R(\phi_{t})|\\
				< &\expectation_{\phi \sim \Phi^{\agentdefaut}}\frac{\varepsilon}{3} = \frac{\varepsilon}{3}
			\end{align}
			Now, let us consider \autoref{agentConvergenceImpliesMeasureConvergenceForEachT} of \autoref{convergingAgentsConvergingDistances}, which demonstrates that for any $t \in \mathbb{N}$ and any sequence of agents $x_{n}$ converging to $\agentdefaut$, $\TVD(\Phi_{t}^{\agentdefaut}, \Phi_{t}^{x_{n}})$ goes to 0 as $n \rightarrow \infty$, so we have
			\begin{align}
				\Forall t \in \mathbb{N} \Exists m \in \mathbb{N} \left(n > m \implies \TVD(\Phi_{t}^{\agentdefaut}, \Phi_{t}^{x_{n}}) < \frac{\varepsilon}{3 \overline{R}}\right).
			\end{align}
			Then we have
			\begin{align}
				\left|\expectation_{\phi \sim \Phi^{x_{n}}}R(\phi_{t}) - \expectation_{\phi \sim \Phi^{\agentdefaut}}R(\phi_{t})\right| < \overline{R} \frac{\varepsilon}{3 \overline{R}} = \frac{\varepsilon}{3}.
			\end{align}
			Thus, for sufficiently large $t$ and $m$, we have for $n > m$
			\renewcommand{\qedsymbol}{\textbf{A.T.C.R.}}
			\begin{align}
				&\left|J(a) - J(x_{n})\right|\\
				= &\left|J(a) - J(x_{n}) + \left(\expectation_{\phi \sim \Phi^{\agentdefaut}}R(\phi_{t}) - \expectation_{\phi \sim \Phi^{\agentdefaut}}R(\phi_{t})\right) + \left(\expectation_{\phi \sim \Phi^{x_{n}}}R(\phi_{t}) - \expectation_{\phi \sim \Phi^{x_{n}}}R(\phi_{t})\right)\right|\\
				= &\left|\left(J(a) - \expectation_{\phi \sim \Phi^{\agentdefaut}}R(\phi_{t})\right) - \left(J(x_{n}) - \expectation_{\phi \sim \Phi^{x_{n}}}R(\phi_{t})\right) + \left(\expectation_{\phi \sim \Phi^{\agentdefaut}}R(\phi_{t}) - \expectation_{\phi \sim \Phi^{x_{n}}}R(\phi_{t})\right)\right|\\
				< &\left|J(a) - \expectation_{\phi \sim \Phi^{\agentdefaut}}R(\phi_{t})\right| + \left|J(x_{n}) - \expectation_{\phi \sim \Phi^{x_{n}}}R(\phi_{t})\right| + \left|\expectation_{\phi \sim \Phi^{\agentdefaut}}R(\phi_{t}) - \expectation_{\phi \sim \Phi^{x_{n}}}R(\phi_{t})\right|\\
				< &\,\,\frac{\varepsilon}{3} + \frac{\varepsilon}{3} + \frac{\varepsilon}{3} = \varepsilon.
			\end{align}		
		\end{proof}
		\renewcommand{\qedsymbol}{\textbf{Q.E.D.}}
	
\section{Conclusion}\label{conclusion}
	In this work we consider the problem of exploration in Reinforcement Learning. We find that exploration is understood and well-defined in the \emph{dynamic paradigm} of Richard Bellman \cite{Bellman1954}, but that it is not well-defined for other optimization paradigms used in Reinforcement Learning. In dynamic Reinforcement Learning, exploration serves to collect the information necessary for dynamic programming, as described in \autoref{incompleteInformation}. In non-dynamic Reinforcement Learning - what we call \emph{naïve} Reinforcement Learning - the situation is more complex. We find that dynamic methods of exploration are effective in naïve methods, but that the explanation of their effect offered by dynamic programming \emph{cannot} explain their efficacy in naïve methods, which do not use the information required by dynamic programming. 
	
	This leads us to several questions: Why are exploration methods designed to provide information useless to naïve Reinforcement Learning nonetheless effective for naïve methods? What should the definition of exploration be for naïve methods? To what extent does this more general kind of exploration contribute to the effectiveness of dynamic exploration in dynamic methods? To resolve these questions, we consult the commonalities of several dynamic methods of exploration, finding two: first, their dynamic justification, and second, their mechanism: considering different agents which are deemed likely to demonstrate different distributions of paths.
	
	Of these, only the mechanism might serve to explain dynamic exploration's efficacy in naïve methods, and we take this mechanism as the definition of naïve exploration. This, definition, however, leaves a gap: under it, totally random experimentation with agents is explorative. This may be effective in small problems, but it is unprincipled. We find a principle in \emph{Novelty Search} \cite{Lehman2011}: in exploration one should consider agents which are \emph{novel} relative to the agents which have already been considered. To determine novelty, they use the distance between the \emph{behavior} of an agent and those considered in the past.
	
	However, we find their notion of novelty deficient for the purpose of defining naïve exploration; they require that function which determines the behavior of an agent be separately and manually determined for each reinforcement learning problem. Fortunately, this view is not held universally in the literature. We consider a cluster of behavior functions which we call \emph{primitive behavior} \cite{Meyerson2016, Parker-Holder2020, Stork2020}. Primitive behavior is powerful: because it is composed of the actions of an agent, it is possible in some processes for primitive behavior to fully determine the distribution of paths, and thus to determine every notion of behavior derived from $(\mathcal{P}, \agentdefaut)$.
	
	Unfortunately, primitive behavior has several flaws. First, only in certain finite processes may the primitive behavior of an agent fully determine $\Phi^{\agentdefaut}$. Second, primitive behavior can inappropriately distinguish between agents (see \autoref{defAgentIdentity}). Third, it necessarily retains the manual selection requirement in decision processes with infinitely many states. In \autoref{pseudometrics}, we describe a more general notion of the distance between agents - one which does not require a behavior function. Instead, we define a structure on the set of agents itself. We call the resulting structure an \emph{agent space}.
	
	In \autoref{theAgentSpace} and \autoref{continuity}, we describe the topology of the agent space, demonstrating that it carries information about many important aspects of Reinforcement Learning, including the distribution of paths produced by an agent and standard formulations of the reward of an agent. Using these facts, we demonstrate that, for many function approximators, reward is a continuous function of the parameters of an agent. 
	
	In a future work \cite{Allen2021}, we use techniques described in \autoref{epilogue} to join the agent space with Novelty Search to perform Reinforcement Learning using a naïve, scalable learning system similar to \emph{Evolution Strategies} \cite{Salimans2017}. We test this method in a variety of processes and find that it performs similarly to ES in problems which require little exploration, and is strictly superior to ES in problems in which exploration is necessary.
	
\separatorpage
\addcontentsline{toc}{section}{References}
\printbibliography
\separatorpage

\ifx\appendixon\True
\begin{appendices}

\section{Exploration Methods}\label{explorationMethods}
	This appendix contains a brief review of the exploration methods mentioned in \autoref{exploration}.
	\subsection{Undirected Exploration}\label{undirected}
		\subsubsection{$\varepsilon$-Greedy}
			One of the simplest undirected exploration algorithms is the $\varepsilon$-greedy algorithm described in \autoref{defEpsilonGreedy}. 
			
			In \autoref{defEpsilonGreedy}, we assumed that $\agentdefaut_{\theta}$ was a real function approximator, but $\varepsilon$-Greedy can be applied to a broader range of intermediates. All that is necessary is that the underlying function approximator $\agentdefaut_{\theta}$ indicate a single action - referred to as the ``greedy'' action, a reference to Reinforcement Learning algorithms which explicitly predict the value of actions (see \autoref{defActionValueFunction}). In such algorithms, the ``greedy action'' is the one which is predicted to have the highest value. We call functions with this property \emph{deterministic}, and the greedy action their deterministic action.
			
			\begin{definition}[$\varepsilon$-Greedy]\label{epsilonGreedy}
				An $\varepsilon$-greedy output function renders a deterministic agent stochastic by changing its action with probability $\varepsilon \in [0,1]$ to one drawn from a uniform distribution over the set of actions, $U$, and retaining the deterministic action with probability $1-\varepsilon$:
				\begin{align}
				O_{\varepsilon}(\agentdefaut_{\theta}(s)) =
				\begin{cases}
				O_{\text{greedy}}(\agentdefaut_{\theta}(s)) & \text{with probability } 1-\varepsilon\\
				U(\actionSet) & \text{with probability }\hspace{0.88em} \varepsilon\\
				\end{cases}
				\end{align}
				where $O_{\text{greedy}}$ is the output function which takes $\agentdefaut_{\theta}$'s greedy action.\footnote{For notational simplicity we assume that the function underlying $\agentdefaut$ is parameterized ($\agentdefaut_{\theta}$). This is not necessary to apply the methods of this section.}
			\end{definition}
			\begin{remark}[$\varepsilon = 0$]
				Notice that the case $\varepsilon = 0$ collapses to the deterministic agent $\agentdefaut_{\theta}$, and the case $\varepsilon = 1$ is the uniformly random agent.
			\end{remark}
		
			The major benefit of $\varepsilon$-greedy sampling is that in a finite decision process, \emph{every} path has a non-zero probability (provided $\varepsilon > 0$). Unfortunately, that can only be accomplished by assigning a diminutive probability to each of those paths. As a path deviates further from the paths generated by the agent $\agentdefaut_{\theta}$, its probability decreases exponentially with each action which deviates from $\agentdefaut_{\theta}$.
			
			That restriction is not necessarily bad for optimization; by visiting paths which require few changes to the actions of $\agentdefaut_{\theta}$, the newly discovered states are \emph{nearly} accessible to $\agentdefaut_{\theta}$, which may make them more poignant to learning algorithms, which typically change locus agents only by small amounts in each epoch.
			
			While $\varepsilon$-Greedy can be applied to processes with finite or infinite sets of actions, the next method, Thompson Sampling, can only be defined for processes with finite sets of actions.
		
		\subsubsection{Thompson Sampling and Related Methods}\label{thompsonSampling}
			Other major undirected exploration methods \edit{find citations for these claims or at least justify them in some way.} operate using a similar mechanism to $\varepsilon$-Greedy, to very different effect. Just like $\varepsilon$-Greedy sampling, Thompson Sampling acts as $O$, taking the range of a function $\agentdefaut_{\theta}$ to the set of probability distributions of actions. Whereas $\varepsilon$-Greedy produces a distribution which varies only in the agent's deterministic\footnote{Different learning algorithms approximate different objects; in $\mathcal{Q}$-Learning (\autoref{incompleteInformation}), $\agentdefaut_{\theta}$ approximates the \emph{action-value} of a state-action pair $(s,\action)$; in policy gradients, its meaning is dependent on the output function $O$.} action, Thompson Sampling produces a distribution which varies with the agent's output for \emph{each} action; unlike $\varepsilon$-Greedy, Thompson Sampling is continuous in the output of $\agentdefaut_{\theta}$.
			
			\begin{definition}[Thompson Sampling]
				A Thompson Sampling output function produces a distribution of actions from the output of a real function approximator $\agentdefaut_{\theta}$. Thompson Sampling requires that $\agentdefaut_{\theta}(s)$ be a real vector of dimension $|\actionSet|$, whose elements are nonnegative and have sum 1; $\agentdefaut_{\theta}(s) : S \rightarrow \{x \in \mathbb{R}^{|\actionSet|} \,|\, x_{i} \geq 0, \sum x_{i} = 1\}$. The Thompson Output Function produces the distribution of actions
				\begin{align}
				O_{\text{Thompson}}(\agentdefaut_{\theta}(s)) =
				\begin{cases}
				\action_{1} & \text{with probability } \agentdefaut_{\theta}(s)_{1},\\
				\action_{2} & \text{with probability } \agentdefaut_{\theta}(s)_{2},\\
				\action_{3} & \text{with probability } \agentdefaut_{\theta}(s)_{3},\\
				\vdots\\
				\action_{|\actionSet|} & \text{with probability } \agentdefaut_{\theta}(s)_{|\actionSet|}.\\
				\end{cases}
				\end{align}
			\end{definition}
			Many function approximators do not naturally produce values which fall in the set of acceptable inputs to $O_{\text{Thompson}}$. Several methods may be employed to rectify these approximators with Thompson Sampling. One common method is known as \emph{Boltzmann Exploration} (or as a \emph{softmax} layer) \cite[37]{Sutton2015}:
			\begin{definition}[Boltzmann Exploration]
				A Boltzmann Exploration output function produces a distribution of actions from the output of a real function approximator $\agentdefaut_{\theta}$. It is most easily understood as a ``pre-processing'' for Thompson Sampling function. Let $\excoef$ be a real parameter ($\excoef$ is sometimes called temperature \cite{Edman2019}). Then,
				\begin{align}
				\text{softmax}&(x): R^{|\actionSet|} \rightarrow \{x \in \mathbb{R}^{|\actionSet|} | x_{i} \geq 0, \sum x_{i} = 1\},\\
				\text{softmax}&(\agentdefaut_{\theta}(s))_{i} = \frac{e^{\agentdefaut_{\theta}(s)_{i}\excoef}}{\sum e^{\agentdefaut_{\theta}(s)_{k}\excoef}}.
				\end{align}
				Which can be composed with the regular Thompson output function:
				\begin{align}
				O_{\text{Boltzmann}}(\agentdefaut_{\theta}(s)) = O_{\text{Thompson}}(\text{softmax}(\agentdefaut_{\theta}(s))).
				\end{align}
			\end{definition}
			Boltzmann Exploration is among the most common methods of creating functions which are compatible with Thompson Sampling because of its beneficial analytical properties: it is continuous, has a simple derivative (especially important for back-propagation), and guarantees that every action has a non-zero probability.
			
			A different kind of augmentation of Thompson Sampling and other stochastic output functions is \emph{Entropy Maximization} \cite{Williams1991}. In contrast with Boltzmann Exploration, entropy maximization modifies the learning process itself through the immediate reward function.
			\begin{definition}[Entropy Maximization]\label{defEntropyMaximization}
				Entropy Maximization is a method used with stochastic agents which adds the conditional \emph{entropy} $h(\action_{t}|\agentdefaut,s_{t})$ of the action with respect to the distribution $\agentdefaut(s)$ to the immediate reward,
				\begin{align}
				r_{\text{entropy}}(\phi_{t}) = r(\phi_{t}) + \excoef h(\action_{t}|\agentdefaut,s_{t}),
				\end{align}
				where $\excoef$ is a positive real parameter of the optimizer \cite{Williams1991}.
			\end{definition}
			These entropy \emph{bonuses} cause the learner to consider both the reward which an agent attains and its propensity to select a diversity of actions. The learner is thus encouraged to consider agents which express greater ``uncertainty'' in their actions, slowing the convergence of the locus agent.
			
			With respect to dynamic exploration, there is little difference between $\varepsilon$-Greedy and Thompson Sampling. Both algorithms explore the process by selecting agents which allow them to experience unexplored aspects of the process. From the naïve perspective, this similarity is overshadowed by a difference in their analytical properties: in finite problems, Thompson Sampling agents act from a \emph{continuous} set of actions, whereas $\varepsilon$-Greedy agents use a [modified] finite set of actions.
			
			The exploration methods discussed in this section are fairly homogeneous, precisely because they are \emph{un}directed; the only way to explore \emph{without} direction is to inject stochasticity into the optimization process. Conversely, the methods of the next section are considerably more diverse; there are many ways to direct an explorative process.
		
		\subsection{Directed Exploration}\label{directed}
			The variety of directed exploration methods make the genre difficult to summarize. Perhaps the simplest description of directed methods as the complement of undirected methods. Undirected exploration methods use exclusively stochastic means to explore; they do not incorporate any information specific to the process. Directed exploration methods thus include any exploration method which \emph{does} incorporate such information \cite{Thrun1992}. This section describes two major families of directed exploration methods: \emph{count-based} and \emph{prediction-based} through a pair of representative methods \cite{Weng2020}. We begin with count-based exploration, exemplified by \#Exploration \cite{Tang2017}.
			
			Count-based methods \cite{Tang2017, Bellemare2016} count the number of times that each state (or state-action pair, see \autoref{incompleteInformation}) has been observed in the course of learning, and use that count to inform the course of learning, to encourage visitation of scarcely visited states. Count-based algorithms have appealing guarantees in finite processes \cite{Bellemare2016}, but lose those guarantees in infinite settings. Despite this, count-based exploration continues to inspire exploration techniques in the infinite setting. \#Exploration is a recent method which discretizes infinite problems, imitating traditional count-based methods.
			
			\begin{definition}[\#Exploration]\label{hashtagExploration}
				\#Exploration \cite{Tang2017} is an algorithm which augments the immediate reward function with an \emph{exploration} bonus in the same manner as the entropy bonus of \autoref{defEntropyMaximization}. However, instead of encouraging the learner to pursue agents which attempt new actions or visit rarely visited states, \#Exploration uses \emph{hash codes}. The hash codes are generated by a hashing function $H(s)$ which discretizes an unmanageable (i.e. large or infinite) set of states into a manageable finite set of hash codes. Using these hash codes as a proxy for states, \#Exploration assigns its exploration bonus in much the same way as a traditional count-based method:\edit{Remember to systematically apply this kind of formatting to the final draft. And then again to those versions which are forced into another form factor.}
				\begin{align}
				r_{\#}(\agentdefaut,s) = r(\agentdefaut, s) + \frac{\excoef}{\sqrt{n(H(s))}}.
				\end{align}
				Here, $r_{\#}$ is the combination of the immediate reward function and the exploration bonus for that state, $n$ is the state-count function, a tally of the number of times that a state with the same hash code as $s$, $H(s)$, has been visited, and $\excoef$ is a positive real number. \#Exploration pursues its goal as a count-based method by assigning greater exploration bonuses to states which have been visited fewer times.
			\end{definition}
			The next class of exploration methods in this section is called \emph{prediction-based} exploration. Whereas count-based methods estimate the new information that an action will collect with a measurement of the learner's experience of each state, prediction-based methods attempt to estimate the \emph{quality}, rather than the mere \emph{quantity} of the collected information. To do this, they employ a separate modeling method which predicts the next state of the process. The better that prediction is, the higher the quality of the information which the learner has about that part of the process.
			
			\begin{definition}\label{stadieDefinition}
				In \emph{Incentivizing Exploration}\cite{Stadie2015}, Stadie et al. estimate the quality of information which the executor has gathered by using that information to train a \emph{dynamics model} $\mathcal{M}$ to estimate the next state $s_{t+1}$ from $(s_{t}, \action_{t})$\footnote{This algorithm uses ``state encodings'', similar to the hash codes of \#Exploration, rather than states.}. They reason that if $\mathcal{M}$ accurately estimates the next state (as measured by the distance between $\mathcal{M}(s_{t}, \action_{t})$ and $s_{t+1}$), then the executor has gathered better information about that state-action pair. Thus, they assign exploration bonuses so as to encourage consideration of agents which visit state-action pairs for which the distance $||\mathcal{M}(s_{t}, \action_{t}) - s_{t+1}||$ is large:
				\begin{align}
				r_{\text{Stadie}}(s_{t},\action_{t}) = r(s_{t},\action_{t}) + \excoef\frac{||\mathcal{M}(s_{t},\action_{t}) - s_{t+1}||}{tC},
				\end{align}
				where $\beta$ is a constant, and $C$ is decay constant (i.e. an increasing function of the learning epoch).
			\end{definition}
	\section{Behavior Functions in the Literature}\label{behaviorFunctions}
		Many of the behavior functions which have been proposed have been influenced by the behavior functions of Lehman and Stanley's initial work, and by their advice on the subject in \emph{Abandoning Objectives: Evolution Through the Search for Novelty Alone}:
		\begin{quote}
			Although generally applicable, novelty search is best suited to domains with deceptive fitness landscapes, intuitive behavioral characterizations, and domain constraints on possible expressible behaviors.\hfill - Lehman and Stanley \cite[200]{Lehman2011}
		\end{quote}
		This passage provides important insight for those who wish to apply Novelty Search to new domains in the tradition of Lehman and Stanley, but their suggestions also make it difficult to analyze Novelty Search independent of the choice of behavior function. This is especially problematic for the open-ended use of Novelty Search; without a general notion of behavior, there are few options for the comparison of behavior functions to one another or their absolute evaluation. In a given decision process, one can compare the outcomes of Novelty Search processes with different behavior functions by considering the diversity of behaviors which they produce, \emph{but this diversity must be measured by one of these or a different notion of behavior}. One could consider each behavior function's propensity to find high-quality agents, but this is a return to the just-abandoned objectives.
		
		Lehman and Stanley are forced to compare their behavior functions in the maze environment along these lines. In a discussion of the degrees of ``conflation'' (assignment of the same behavior to different agents) present in their behavior functions: ``[I]f some dimensions of behavior are completely orthogonal to the objective, it is likely better to conflate all such orthogonal dimensions of behavior together rather than explore them.'' To address these issues and make it possible to apply Novelty Search to a wider range of decision processes, several authors have considered \emph{general} behavior functions \cite{Gomez2009, Conti2017, Meyerson2016, Parker-Holder2020, Stork2020}.
		
		The rest of this appendix provides a brief survey of some behavior functions present in the literature, beginning with the specific functions of \emph{Abandoning Objectives} \cite{Lehman2011}, and proceeding to general behavior functions, including those of \cite{Gomez2009, Conti2017}.
		
		\subsection{The Behavior Functions of \emph{Abandoning Objectives}}
		The main decision processes in \emph{Abandoning Objectives} are two-dimensional mazes. They consider several behavior functions in this environment, all of which are based on the position of the agent. The primary behavior function they consider is what we call the \emph{final position behavior function}:
		
		\begin{definition}[Final Position Behavior]
			\begin{align}
			\behavior(\agentdefaut) = p_{t\textsubscript{max}}
			\end{align}
			Where $p_{t\textsubscript{max}}$ is the position at the final time in a sampled truncated trajectory $\phi_{t_{\max}}$.
		\end{definition}
		They consider another positional behavior function: the position of the agent over time.	
		\begin{definition}[Position Over Time Behavior]
			\begin{align}
			\behavior(\agentdefaut) = (p_{t_{1}}, p_{t_{2}}, ... , p_{t_{N-1}}, p_{t_{N}})
			\end{align}
			For $0 \leq t_{i} < t_{i + 1} \leq t_{\max}$
		\end{definition}
		
		As noted in \autoref{notAMetric}, while these are functions into $\mathbb{R}^{2}$ in the case of final position behavior, and $\mathbb{R}^{2N}$ in the case of position over time behavior, neither $\mathbb{R}^{2}$ nor $\mathbb{R}^{2N}$ are treated as metric spaces. Instead, both of these are equipped with a symmetric \cite[23]{Arkhangelskii1990}: the square of the usual Euclidean distance.
		
		These examples reveal the intuitions about behavior which Lehman and Stanley relied upon to implement Novelty Search. First, rather than reflecting the actions of an agent alone, both of these behavior functions reflect the results of the agent's interaction with the process - in fact, they reflect the position of the agent, a function of the state. Second, these behavior functions are distilled, considering only one or a few points of time $t_{i}$.
		
		In the other environment, Biped Locomotion, Lehman and Stanley take a different approach to selecting the times $t_{i}$, opting to collect spatial information once per simulated second. Explaining that difference, they write: ``Unlike in the maze domain, temporal sampling is necessary because the temporal pattern is fundamental to walking.'' This is a strange argument, since reinforcement learning problems are \emph{defined} by their temporality (see \autoref{defDecisionProcess}).
		
		\subsection{General Behavior Functions}
		Since the publication of \cite{Lehman2008}, Lehman and Stanley's first paper on Novelty Search, many authors have sought general notions of behavior \cite{Gomez2009, Conti2017, Meyerson2016, Parker-Holder2020, Stork2020}. This section analyzes several of these behavior functions. Let us begin with a simple notion of distance on the set of agents which is defined with for any method using a parameterized agent:
		\begin{example}[The distance of $\theta$]
			Consider two agents, $\agentun$ and $\agentdeux$, represented by function approximators of the same form. Assume that they are parameterized by an ordered list of real numbers, and let their parameters be $\theta_{\agentun}$ and $\theta_{\agentdeux}$. Then,
			\begin{align}
				\behavior(\agentdefaut_{\theta_{\agentdefaut}}) = \theta_{\agentdefaut}
			\end{align}
			is a behavior function and
			\begin{align}
			d(\behavior(\agentun),\behavior(\agentdeux)) = ||\theta_{\agentun} - \theta_{\agentdeux}||
			\end{align}
			is a metric on this set of behaviors.
		\end{example}
		
		Although it is a metric (though it does not satisfy the indiscernibility of identicals under the quotient operation of \autoref{defAgentIdentity}), this distance is unsatisfactory in several ways. For example, it can assign an agent a non-zero distance from itself if the agent can be parameterized by two different sets of parameters (see \autoref{daab0} and \cite{Parker-Holder2020}).
		
		An early work of Gomez et al. \cite{Gomez2009} introduced a behavior function which maps agents to a concatenation of truncated trajectories. They then use the \emph{normalized compression distance} (NCD) as a metric on this set of finite sequences.
		\begin{definition}[Gomez et al. Behavior]
			Gomez et al. \cite{Gomez2009} define the behavior of an agent as a concatenation of a number of observed truncated paths, 
			\begin{align}
			\behavior(\agentdefaut) = \phi_{\agentdefaut,1}^{n_{1}}|\phi_{\agentdefaut,2}^{n_{2}}|...
			\end{align}
			
			As the distance on this set of behaviors $\behavior(\agentSet)$, Gomez uses the \emph{normalized compression distance} $\text{NCD}(\behavior(\agentun), \behavior(\agentdeux))$, which is an approximation of the mutual information of a pair of strings:
			\begin{align}
			\text{NCD}(\agentun,\agentdeux) = \frac{C(\behavior(\agentun)|\behavior(\agentdeux)) - \min\{C(\behavior(\agentun)), C(\behavior(\agentdeux))\}}{\max\{C(\behavior(\agentun)), C(\behavior(\agentdeux))\}}.
			\end{align}
			Where $C$ is the length of the compressed sequence.
		\end{definition}
		
	\section{Agent Spaces In Practice}\label{epilogue}
		While the agent space is in general not a metric space (see \autoref{pseudometrics}, \autoref{dcab0}, and \autoref{equivalentDistances}), this does not proscribe its use in e.g. Novelty Search, which has long used metric-adjacent spaces (see \autoref{notAMetric}). In an upcoming work \cite{Allen2021}, we describe an approach to the Reinforcement Learning problem based on an extension of Evolution Strategies \cite{Salimans2017} which combines naïve reward and Novelty Search \cite{Lehman2011} of the agent space, selecting the locus agent as the perspective for comparisons during each epoch, to solve a variety of reinforcement learning problems (see \autoref{implementationSketch}).
		
		Novelty Search and the Agent Space to be naïve artifacts, but we cannot a priori restrict them to naïve learning methods. For example, \cite{Stork2020} uses several versions of primitive behavior (see \autoref{primitiveBehavior}) to explain ``reward behavior correlation[s]'', showing that agents which are similar under certain primitive behavior functions perform similarly. In \autoref{continuity} we demonstrate that the Agent Space completes this line of inquiry by demonstrating analytically that reward is a continuous function of the agent in the agent space. The reasoning of the Agent Space in \autoref{defAgentIdentity} also provides a clean explanation for the observation of \cite{Stork2020} that certain states may be totally unimportant to performance.
		
		\subsection{When $d_{a}$ is Equivalent to $d_{b}$}\label{equivalentDistances}
			While local distances do not generally produce homeomorphic topologies, it is worthwhile to note that many basic problems in the literature \emph{do} have agents which produce homeomorphic topologies, especially in the Markov and strictly Markov cases. Let us begin by considering the equivalence of the measures underlying local distances:
			\begin{definition}[Equivalence of Measures]
				A pair of measures $\mu$, $\nu$ each on a measurable space $(M, \Sigma)$ are said to be \emph{equivalent} iff \cite[157]{Klenke}
				\begin{align}
				\Forall X \in \Sigma(\mu(X) = 0 \iff \nu(X) = 0)
				\end{align}
			\end{definition}
			When the measures underlying $d_{\agentun}$ and $d_{\agentdeux}$, $\Phi^{\agentun}$ and $\Phi^{\agentdeux}$, are equivalent, they produce equivalent topologies. In general, this is rare. In most decision processes, some paths can only be visited by a subset of agents. However, there are certain circumstances where these distances are guaranteed to be equivalent.
			
			Clearly, if two agents differ with non-zero probability, then the distributions of \emph{truncated paths} which they produce must also differ. However, this does not apply when only Markov agents are considered (see \autoref{defMarkovProperty}). In this case, if the distributions of \emph{states}, rather than truncated paths, are equivalent, then the distances produce equivalent topologies.
			
			\begin{remark}[Notation for the Probability of a State]
				The next few results require a simple notation for the probability of a state occurring in a distribution of paths. This is complicated by the fact that in any path there are an infinity of states, so that the sum of the probabilities of a state occurring at each time $t \in \mathbb{N}$ might be infinite. To resolve this, we weight the probability of a state at $t$ by $\omega(t)$ and then normalize these probabilities with $\Omega$. Let
				\begin{align}
					\mathbb{P}(s \,|\, \agentdefaut) = \frac{\sum_{t \in \mathbb{N}}\mathbb{P}(\phi_{t}(s) = s)\omega(t)}{\Omega}.
				\end{align}
			\end{remark}
			In general, this gives
			\begin{theorem}[A Condition for the Equivalence of $d_{\agentun}$ and $d_{\agentdeux}$]
				In a Markov decision process with a finite set of states, the local distances $d_{\agentun}$ and $d_{\agentdeux}$ are equivalent whenever the distributions of states in $\Phi^{\agentun}$ and $\Phi^{\agentdeux}$ are equivalent.
			\end{theorem}
			This theorem has several important manifestations. Let us consider the case where the local distances of all agents possible in the process are equivalent:
			\begin{lemma}[Conditions for the Equivalence of all Local Distances]
				The most general condition for the equivalence of all local distances is
				\begin{align}
					\Forall \agentdefaut \in \agentSet \Forall s \in S (\mathbb{P}(s \,|\, \agentdefaut) > 0).
				\end{align}
				There are two common conditions which are more specific but easier to verify, which may help with the application of this result. First, if every transition probability is greater than zero, then certainly the probability of each state in the distribution of paths is greater than 0:
				\begin{align}
					\Forall \phi_{t} \in \bigcup_{t \in \mathbb{N}}\Phi_{t} (\mathbb{P}(s \,|\, \sigma(\phi_{t})) > 0).
				\end{align}
				Even more specifically, but also more easy to test, if the probability of each state at the beginning of the process is non-zero, then the probability of each state in the distribution of paths is greater than~0:
				\begin{align}
					\Forall s \in S (\mathbb{P}(s \,|\, \sigma_{0})).
				\end{align}
			\end{lemma}
			Of course, this result is of little importance if these local distances cease to be useful for the analysis of the underlying decision process. By construction, these equivalent local distances are relevant only to Markov processes. Importantly, these local distances cannot detect differences in distributions of paths which are not caused by differences in distributions of states, or more specifically of state-action pairs. Thus, these local distances cannot guarantee the continuity all of the functions considered in \autoref{continuity}, but they \emph{do} apply to the cumulative reward functions described by \eqref{sumReward}.
		
		\subsection{Deterministic Agents}
			The theorems of \autoref{theAgentSpace} rely on \emph{stochastic} agents to justify the topology of the Agent Space. This reliance stems from the fact that we are concerned in that section primarily with distributions of \emph{paths}. Because paths consist of sequences of states \emph{and actions}, distributions of paths can only approach one another (with respect to total variation) if, in response to a single truncated prime path $\phi^{\prime}_{t}$, a distribution of actions is taken.
			
			However, this does not mean that the agent space is useless when deterministic agents are considered. For example, the distributions of states mentioned in \autoref{equivalentDistances} may be continuous even in an agent space composed entirely of deterministic agents, via the stochasticity of the state-transition function. Thus, if the set of actions is connected and for every truncated prime path $\phi^{\prime}_{t}$ the state-transition distribution is a continuous function of the final action $\action_{t}$, then the distributions of states are continuous in the agent space. Then, an immediate reward function which considers only the state would also be continuous in the agent space.
			
		\subsection{Sketch of the Novelty Search Methods of \cite{Allen2021}}\label{implementationSketch}
			In an upcoming work, we use the Agent Space in conjunction with Novelty Search to develop a distributed optimization algorithm for Reinforcement Learning problems. Our method evaluates candidate agents with a path collected by the locus agent at each training epoch, resulting in a non-stationary objective that encourages the agent to behave in ways that it has not yet behaved, on states that it can currently encounter. The following pseudo-code is a summary of this method.
			\begin{algorithm}[H]
				\caption{Strategy \& Reward Optimization Algorithm}
				\begin{algorithmic}[1]
					\For {each epoch}
					\State Set batch $\agentSet_{t} := \emptyset$
					\For {desired batch size}
					\State Set $\varepsilon \sim \mathcal{N}(0, I)$
					\State Gather a path $\phi$ and cumulative reward $R(\phi)$ with $\agentdefaut_{t} + \varepsilon$
					\State Set $N(\agentdefaut_{t} + \varepsilon) \gets \underset{0 < i \leq t}{\min} [d_{a_{t}}(\agentdefaut_{t} + \varepsilon, \agentdefaut_{t-i})]$ for prior epochs $t-i$
					\State Append $(R(\phi), N(\agentdefaut_{t} + \varepsilon))$ to $\agentSet_{t}$
					\EndFor
					\State Compute $\nabla_{\agentdefaut_{t}} G(\agentdefaut_{t}) := \nabla_{\agentdefaut_{t}} R(\agentdefaut_{t}) + \nabla_{\agentdefaut_{t}} N(\agentdefaut_{t})$ via Finite Differences with $\agentSet_{t}$.
					\State Update $\agentdefaut_{t}$ by following $\nabla_{\agentdefaut_{t}} G(\agentdefaut_{t})$
					\EndFor
				\end{algorithmic}
			\end{algorithm}
			
			We approximate $d_{\agentdefaut_{t}}(\agentdefaut_{t} + \varepsilon, \agentdefaut_{t-i})$ by evaluating candidate agents on a set of states that we gather every epoch. To do this, we follow the agent $\agentdefaut_{t}$ in the decision process until $K$ total states have been encountered, then store them in a set denoted $\zeta$. The distance $d_{\agentdefaut_{t}}(\agentdefaut_{t} + \varepsilon, \agentdefaut_{t-n})$ can then be approximated by evaluating the responses of $\agentdefaut_{t} + \varepsilon$ and $\agentdefaut_{t-i}$ on only $\zeta$, rather than by integration on $\Phi^{\agentdefaut_{t}}$.
\end{appendices}
\fi
\end{document}